\documentclass[conference,letterpaper]{IEEEtran}
\newif\iffullversion
\fullversiontrue
\usepackage[T1]{fontenc}
\usepackage[utf8]{inputenc}
\usepackage{amsmath,amsfonts,amssymb,amsthm,bm}
\usepackage{dsfont}
\usepackage{cite}
\usepackage{xspace,xcolor}
\usepackage[noend]{algpseudocode}
\usepackage{enumitem}
\usepackage{thmtools,thm-restate}
\usepackage[hidelinks,hypertexnames=false]{hyperref}
\usepackage{cleveref}
\providecommand{\halmos}{\qedhere}
\newenvironment{algorithmfigure}[1][t]{%
  \begin{figure}[#1]%
  \crefalias{figure}{algorithm}%
}{\end{figure}}
\crefname{algorithm}{Algorithm}{Algorithms}
\Crefname{algorithm}{Algorithm}{Algorithms}
\NewDocumentEnvironment{algorithmframe}{+b}{%
  \begingroup
  \setlength{\fboxsep}{4pt}%
  \setlength{\fboxrule}{0.4pt}%
  \noindent\fbox{%
    \begin{minipage}{\dimexpr\linewidth-2\fboxsep-2\fboxrule\relax}
      \raggedright
      #1
    \end{minipage}%
  }%
  \par\endgroup
}{}
\makeatletter
\newcommand{\LineCont}{\Statex\hspace*{\dimexpr\ALG@thistlm+1em\relax}}
\makeatother
\AddToHook{cmd/appendices/after}{%
  \crefalias{section}{appendix}%
  \crefalias{subsection}{subappendix}%
  \crefalias{subsubsection}{subsubappendix}%
}
\crefname{appendix}{Appendix}{Appendices}
\Crefname{appendix}{Appendix}{Appendices}
\crefname{subappendix}{Appendix}{Appendices}
\Crefname{subappendix}{Appendix}{Appendices}
\crefname{subsubappendix}{Appendix}{Appendices}
\Crefname{subsubappendix}{Appendix}{Appendices}
\NewDocumentCommand{\fullversionref}{o m}{%
  \iffullversion
    \IfValueT{#1}{\Cref{#1} in }\Cref{#2}%
  \else
    \hyperlink{wcmdp:full-version}{the full version}%
  \fi
}
\hypersetup{pdftitle={Achieving an O(1/N) Optimality Gap in Average-Reward Weakly-Coupled MDPs},
pdfauthor={Yige Hong, Xiangcheng Zhang, Qiaomin Xie, Yudong Chen, Weina Wang}}
 \theoremstyle{plain}
 \newtheorem{theorem}{Theorem}
 \newtheorem*{theorem*}{Theorem}
 
 \newtheorem{lemma}{Lemma}
 \newtheorem{proposition}{Proposition}

 \theoremstyle{definition}
 \newtheorem{definition}{Definition}
\newtheorem{assumption}{Assumption}

\theoremstyle{remark}
\newtheorem{remark}{\textit{Remark}}
\newtheorem*{remark*}{Remark}

\newcommand{\yigecomment}[1]{}%

\providecommand{\hcref}[1]{%
  \texorpdfstring{\protect\NoHyper\Cref{#1}\protect\endNoHyper}{#1}%
}

\newcommand{\norm}[1]{\left\lVert#1\right\rVert}

\newcommand{\normbig}[1]{\big\lVert#1\big\rVert}

\newcommand{\E}[1]{\mathbb{E}\left[#1\right]}

\newcommand{\Ebig}[1]{\mathbb{E}\big[#1\big]}
\newcommand{\EBig}[1]{\mathbb{E}\Big[#1\Big]}

\newcommand{\given}{\,\middle|\,}
\newcommand{\givenplain}{\,|\,}
\newcommand{\givenbig}{\,\big|\,}
\newcommand{\givenBig}{\,\Big|\,}

\newcommand{\Var}[1]{\text{Var}{\left[#1\right]}}

\newcommand{\Probbig}[1]{\mathbb{P}\big[#1\big]}
\newcommand{\ProbBig}[1]{\mathbb{P}\Big[#1\Big]}
\newcommand{\Prob}[1]{\mathbb{P}\left[#1\right]}
\newcommand{\R}{\mathbb{R}}

\newcommand{\indibrac}[1]{\mathds{1}\!\left\{#1\right\}}

\newcommand{\abs}[1]{\left\lvert#1\right\rvert}

\newcommand{\absBig}[1]{\Big\lvert#1\Big\rvert}

\newcommand{\floor}[1]{\left\lfloor#1\right\rfloor}

\newcommand{\ve}[1]{\bm{#1}}

\newcommand{\vone}{\mathds{1}}

\newcommand{\veS}{\ve{S}}

\newcommand{\veA}{\ve{A}}
\newcommand{\sspa}{\mathbb{S}}
\newcommand{\aspa}{\mathbb{A}}
\newcommand{\sumN}{\sum_{i\in[N]}}
\newcommand{\sumsa}{\sum_{s\in\sspa, a\in\aspa}}

\newcommand{\rel}{\textup{rel}} %
\newcommand{\rmax}{r_{\max}}

\newcommand{\ravg}{R}
\newcommand{\rsysn}{\ravg(\pi, \veS_0)}
\newcommand{\rliminf}{\ravg^{-}(\pi, \veS_0)}
\newcommand{\rlimsup}{\ravg^{+}(\pi, \veS_0)}
\newcommand{\ropt}{\ravg^*(N, \veS_0)}
\newcommand{\rrel}{\ravg^\rel}

\newcommand{\pibar}{{\sysbar{\pi}}}
\newcommand{\pibs}{{\pibar^*}}

\newcommand{\lpfp}{\textup{LP fixed-point}\xspace}
\newcommand{\frcontrol}{\textup{$\pibs$-Fixed-Ratio Control}\xspace}

\newcommand{\sysbar}[1]{\bar{#1}}

\newcommand{\simplex}{\Delta}
\newcommand{\threshbar}{\overline{\eta}}

\newcommand{\sempty}{S^{\emptyset}}

\newcommand{\md}{m_d}
\newcommand{\Md}[1]{N\md(x)}

\newcommand{\statdist}{\mu^*}

\newcommand{\wmat}{W}
\newcommand{\umat}{U}

\newcommand{\rhoFinal}{\rho_1}

\newcommand{\errtol}{\epsilon_N^{\text{rd}}}
\newcommand{\errrew}{\epsilon_N^{\text{rew}}}  %

\newcommand{\hw}{h_\wmat}

\newcommand{\hu}{h_\umat}

\newcommand{\lamw}{\lambda_\wmat}
\newcommand{\lamu}{\lambda_\umat}
\newcommand{\rhow}{\rho_{w}}
\newcommand{\rhou}{\rho_{u}}

\newcommand{\KhwOne}{K_{W,1}}
\newcommand{\KhwTwo}{K_{W,2}}
\newcommand{\Chw}{C_{W}}
\newcommand{\KhuOne}{K_{U,1}}
\newcommand{\KhuTwo}{K_{U,2}}
\newcommand{\Chu}{C_{U}}
\newcommand{\KnsOne}{K_{D,1}}
\newcommand{\KnsTwo}{K_{D,2}}
\newcommand{\Cns}{C_{D}}
\newcommand{\KscOne}{K_{C,1}}
\newcommand{\KscTwo}{K_{C,2}}

\newcommand{\Da}{D^{\pibs}}
\newcommand{\Db}{D^{\textup{LP}}}

\newcommand{\Fullstate}{\Sigma}
\newcommand{\fullstate}{\sigma}
\newcommand{\fullssp}{\mathbb{X}}
\newcommand{\loworder}{O_N}
\newcommand{\Vds}{V_1} %
\newcommand{\Vll}{V_2} %
\newcommand{\rhat}{\hat{r}}
\newcommand{\gvec}{\bm{g}}
\newcommand{\rvec}{\bm{r}}
\newcommand{\KVds}{K_1}  %
\newcommand{\KVdst}{K_2}
\newcommand{\Kdrift}{K_{\mathrm{drift}}} %
\newcommand{\CVds}{C}  %
\newcommand{\Kla}{K_{\locapprox}}  %
\newcommand{\Kg}{K_{g}}  %
\newcommand{\Lds}{L_1}
\newcommand{\Lvll}{L_2}
\newcommand{\locapprox}{\Psi} %

\newcommand{\drftup}{u}

\newcommand{\lamq}{\lambda_{Q}}
\newcommand{\sneu}{\tilde{s}}

\newcommand{\slk}{\delta}
\newcommand{\slkb}{\delta}

\Crefname{assumption}{Assumption}{Assumptions}
\crefname{assumption}{assumption}{assumptions}

\newcommand{\Uset}{\mathcal{U}^*}
\newcommand{\Kset}{\mathcal{K}^*}
\newcommand{\Cmat}{C^*}
\newcommand{\Cinv}{C^+}
\newcommand{\Cinvs}{C^+_\sspa}
\newcommand{\opnorm}[1]{\norm{#1}}

\IEEEoverridecommandlockouts
\title{Achieving an $O(1/N)$ Optimality Gap in Average-Reward Weakly-Coupled MDPs}
\author{
\IEEEauthorblockN{Yige Hong\IEEEauthorrefmark{1}\IEEEauthorrefmark{6}, Xiangcheng Zhang\IEEEauthorrefmark{2}\IEEEauthorrefmark{6}, Qiaomin Xie\IEEEauthorrefmark{3},
Yudong Chen\IEEEauthorrefmark{4}, and Weina Wang\IEEEauthorrefmark{5}}
\IEEEauthorblockA{\IEEEauthorrefmark{1}H.\ Milton Stewart School of Industrial \& Systems Engineering, Georgia Institute of Technology\\
yhong320@gatech.edu}
\IEEEauthorblockA{\IEEEauthorrefmark{2}John A.\ Paulson School of Engineering and Applied Sciences, Harvard University\\
xiangchengzhang@fas.harvard.edu}
\IEEEauthorblockA{\IEEEauthorrefmark{3}Department of Industrial and Systems Engineering, University of Wisconsin--Madison\\
qiaomin.xie@wisc.edu}
\IEEEauthorblockA{\IEEEauthorrefmark{4}Department of Computer Sciences, University of Wisconsin--Madison\\
yudongchen@cs.wisc.edu}
\IEEEauthorblockA{\IEEEauthorrefmark{5}Computer Science Department, Carnegie Mellon University\\
weinaw@cs.cmu.edu}
\thanks{\IEEEauthorrefmark{6}This work was conducted while Yige Hong was a PhD student at Carnegie Mellon University and Xiangcheng Zhang was visiting Carnegie Mellon University.}
\thanks{Yige Hong and Weina Wang are supported in part by the NSF Grants ECCS-2145713, CCF-2403194, CCF-2428569, and ECCS-2432545. Yudong Chen is supported in part by NSF grant CCF-2233152 and a Vilas Associates Award. Qiaomin Xie is supported in part by the NSF Grants ECCS-2339794 and ECCS-2432546.}}

\newcommand{\fullversioncitation}{\textbf{[arXiv link pending]}}

\begin{document}
\maketitle
\pagestyle{plain}
\thispagestyle{plain}
\begin{abstract}
We study average-reward weakly-coupled Markov decision processes (WCMDPs), where a WCMDP consists of $N$ smaller MDPs, called arms, that share multiple per-step budget constraints.
We consider the setting where the arms have identical model parameters, multiple actions, and state- and action-dependent costs.
For restless bandits (RBs), a well-studied special case of WCMDPs, prior work has developed policies that achieve an $O(1/\sqrt{N})$ optimality gap under general conditions, and has further identified conditions under which policies can achieve a better-than-$1/\sqrt{N}$ optimality gap.
However, for general WCMDPs, no prior result achieves an optimality gap better than $1/\sqrt{N}$.
In this paper, we identify conditions analogous to those for RBs under which a better-than-$1/\sqrt{N}$ optimality gap is achievable, and design a policy that attains an $O(1/N)$ optimality gap.
Notably, unlike prior approaches based on generalizing priority orderings, our policy is not priority-based but rather is designed to induce locally linear mean-field dynamics. 

\end{abstract}

\iffullversion\else
\noindent\hypertarget{wcmdp:full-version}{\textit{Full version with detailed proofs:}} \fullversioncitation.
\par\smallskip
\fi

\section{Introduction}
\subsection{Motivation}
Weakly-coupled Markov decision processes (WCMDPs) \cite{Haw03_Lagrangian} model sequential resource allocation among many interacting components.
A WCMDP consists of $N$ smaller Markov decision processes, called \emph{arms}, whose state transitions are independent conditional on their current states and chosen actions.
The coupling arises through shared budget constraints: each action incurs costs of one or more types, and the total cost of each type across all arms must remain within its budget at every time step.
This structure appears in applications such as online advertising \cite{BouLu_16_budget}, healthcare resource allocation \cite{biswas2021learning}, surveillance \cite{villar2016indexability}, and machine maintenance \cite{GlaMitAns_05_rb_repair}.
With known model parameters, we maximize long-run average reward per arm. The \emph{optimality gap} is the difference between the optimal reward per arm and that achieved by a policy; we study its order as $N$ grows.

A widely studied special case is the \emph{restless bandit} (RB) problem \cite{Whi_88_rb}.
Each arm has two actions, active and passive, and a single budget limits the number of arms that can be activated at each time step.
An arm's state can evolve under either action.
For a broad review of RBs, their applications, and index policies, see Ni\~no-Mora \cite{Nin_23}.
A substantial literature establishes asymptotically optimal policies, whose optimality gaps vanish as $N\to\infty$.
Results include $o(1)$ gaps \cite{WebWei_90,Ver_16_verloop,Yan_24_multichain} and $O(1/\sqrt{N})$ gaps \cite{HonXieCheWan_23,HonXieCheWan_24}, under different structural assumptions.

For RBs, gaps of smaller order than $1/\sqrt{N}$ are already possible under suitable assumptions.
The papers \cite{GasGauYan_23_whittles,GasGauYan_23_exponential} establish $O(\exp(-CN))$ gaps for suitable priority policies derived from a linear programming (LP) relaxation, where $C>0$ is a constant independent of $N$.
The key assumptions are the aperiodic unichain condition, non-degeneracy, and a uniform global attractor property (UGAP).
Our prior work \cite{HonXieCheWan_24_exp} achieves the same exponential order using a two-set policy, replacing UGAP with a weaker and easier-to-verify local stability condition near an optimal stationary distribution.

These RB results motivate seeking similarly small optimality gaps for general WCMDPs, which allow more than two actions, multiple budget constraints, and costs that depend on both state and action.
Prior work on average-reward WCMDPs establishes $o(1)$ gaps for the special case with multiple actions but a single budget \cite{HodGla_15,Ver_16_verloop,XioWanLi_22}, and for WCMDPs with multiple actions and budget constraints \cite{GolAvr_24_wcmdp_multichain}.
More recent work \cite{ZhaHonWan_25_het} establishes an $O(1/\sqrt{N})$ gap for WCMDPs with multiple actions, multiple budget constraints, and fully heterogeneous arms, whose model parameters may differ across arms.
This guarantee also applies to homogeneous systems, where all arms share the same model parameters.
However, to our knowledge, an optimality gap of smaller order than $1/\sqrt{N}$ has not been established for average-reward WCMDPs with multiple budget constraints and state- and action-dependent costs.

This paper establishes an $O(1/N)$ optimality gap for average-reward WCMDPs with homogeneous arms, multiple actions, multiple budget constraints, and state-dependent costs.
We assume the aperiodic unichain condition, non-degeneracy, and local stability; the last two are additional assumptions relative to the $O(1/\sqrt{N})$ guarantee in \cite{ZhaHonWan_25_het}.
We achieve this bound by extending the two-set policy of \cite{HonXieCheWan_24_exp} from RBs to WCMDPs.
We do not require a global attractor assumption.

\subsection{Technical insight}
When generalizing from RBs to WCMDPs, much of the prior work has focused on adapting the notion of state priorities to define priority policies \cite{HodGla_15,Ver_16_verloop,XioWanLi_22}.
These policies assign priorities to states and favor arms in higher-priority states when choosing actions that incur positive costs.
This approach is natural when there is a single budget constraint and the costs are state independent.
However, with multiple budget constraints and state-dependent costs, a single priority order no longer directly specifies how to balance the competing budget requirements.

In this paper, our main technical insight is that the priority structure itself is not what enables an optimality gap smaller than $1/\sqrt{N}$.
Rather, the key contributor is the \emph{local linearity} of the mean-field dynamics induced by the policy.
Therefore, to generalize this approach to WCMDPs and design a policy that achieves an optimality gap smaller than $1/\sqrt{N}$, we seek a policy whose mean-field dynamics is locally linear.
We elaborate on this below.

For an RB, consider the empirical state distribution that records the fraction of arms in each state, and its mean-field approximation where the random transitions are replaced by their conditional means.
Let $\statdist$ denote the stationary state distribution prescribed by an optimal LP solution; it is a fixed point of the mean-field dynamics under these priority policies.
We can bound the optimality gap by the expected distance between the empirical state distribution and $\statdist$. The empirical state distribution typically fluctuates on the $1/\sqrt{N}$ scale suggested by the Central Limit Theorem (CLT) for conditionally independent arm transitions.
However, prior work \cite{GasGauYan_23_whittles,GasGauYan_23_exponential} shows that, near $\statdist$, the reward and mean transition depend linearly on the empirical state distribution. This observation allows the analysis to compare the \emph{expected} empirical state distribution with $\statdist$. Since random deviations can cancel in expectation, this difference can be of smaller order than $1/\sqrt{N}$.
Together with concentration near $\statdist$, local linearity therefore allows optimality gaps smaller than the CLT fluctuation scale.

In this paper, we extend the local linearity structure directly, retaining the two-set approach of \cite{HonXieCheWan_24_exp} without forcing a priority policy.
The challenge is to preserve local linearity while satisfying multiple budget constraints with state-dependent costs.
We construct an auxiliary system of linear equations (\Cref{eq:wcmdp:active-system}) to determine target state-action frequencies (the desired fractions of all arms in each state-action pair).
When the empirical state distribution is sufficiently close to $\statdist$, this system yields feasible target frequencies that depend linearly on that distribution.
Rounding the target frequencies to feasible integer action counts introduces an $O(1/N)$ residual in the reward and mean transition under our construction.
The rounding residual contributes the $O(1/N)$ term in our gap bound, which is smaller than the $O(1/\sqrt{N})$ gap previously established for general average-reward WCMDPs.

\subsection{Additional related work}
Model predictive control (MPC) provides another approach to the average-reward setting by repeatedly solving a finite-horizon LP and implementing only the first action allocation.
Under a mixing assumption, Gast and Narasimha \cite{GasNar_25_mpc_RB} bound the optimality gap for homogeneous RBs by the sum of a planning-error term that can be reduced by increasing the horizon and an $O(1/\sqrt{N})$ term.
The latter term becomes exponentially small under additional non-degeneracy, LP uniqueness, and local stability conditions, assuming an integer activation budget.
They also discuss a general WCMDP extension, without proving the sharper bound in that setting.
Narasimha and Gast \cite{NarGas_25_mpc_het} study fully heterogeneous systems; their WCMDP guarantee assumes a conjecture that the horizon needed for a fixed planning-error tolerance can be bounded independently of $N$.

So far we have discussed prior work under the long-run average-reward criterion.
For a fixed finite horizon, with performance measured by expected total reward per arm over that horizon, optimality gaps of smaller order than $1/\sqrt{N}$ are available under suitable non-degeneracy conditions.
For homogeneous RBs, Zhang and Frazier \cite{ZhaFra_21} obtain an $O(1/N)$ gap, which Gast, Gaujal, and Yan \cite{GasGauYan_23_exponential} improve to $O(\exp(-CN))$ with an integer activation budget.
For general WCMDPs, $O(1/N)$ guarantees cover both homogeneous systems \cite{GasGauYan_24_reopt} and systems with typed heterogeneity, which permits a fixed number of distinct arm types as $N$ grows \cite{BroZha_23_ftva_and_reopt}.
Zhang \cite{Zha_24_finite_het_nondegeneracy} further develops bounds for fully heterogeneous systems that improve with the degree of non-degeneracy.
Recently, Yan, Wang, and Ying \cite{CheWanYin_25_RB_Gaussian_approx} also establish a $\widetilde{O}(1/N)$ gap for RBs without requiring non-degeneracy conditions, where $\widetilde{O}$ suppresses logarithmic factors in $N$.
These fixed-horizon guarantees can not be directly used to establish an average-reward guarantee: the bounds can grow faster than linearly with the horizon, and the policies themselves depend on the horizon.
Even when additional assumptions give bounds with only linear horizon dependence, the policies still solve optimization problems whose sizes grow with the horizon \cite{GasGauYan_24_reopt,BroZha_23_ftva_and_reopt}.

The LP-update policy of \cite{GasGauYan_24_reopt} inspires our design.
Both approaches use LP solutions to construct target state-action frequencies and exploit local linearity in the current state distribution.
Their LP plans the remaining finite-horizon mean-field evolution; we solve for stationary state-action frequencies and then use \Cref{eq:wcmdp:active-system}, without planning a future trajectory.

\section{Problem formulation}\label{sec:wcmdp:setup}
We consider a WCMDP that consists of $N$ homogeneous Markov decision processes (MDPs), indexed by $[N]=\{1,\ldots,N\}$, each specified by $(\sspa,\aspa,P,r)$. The state space $\sspa$ and action space $\aspa=\{0,1,\ldots,|\aspa|-1\}$ are finite. The transition kernel is $P:\sspa\times\aspa\times\sspa\to[0,1]$, and $r:\sspa\times\aspa\to\R$ is the reward function. We call each MDP an \emph{arm} and set $\rmax=\max_{s,a}|r(s,a)|$. Model parameters are known. Conditional on the current states and chosen actions, arms transition independently according to $P$.
A policy $\pi$ may be randomized and history-dependent. We write $\veS_t^\pi=(S_t^\pi(i))_{i\in[N]}$ and $\veA_t^\pi=(A_t^\pi(i))_{i\in[N]}$ for the state and action vectors.

The WCMDP system has $K$ budget constraints, specified by cost functions $c_k\colon \sspa\times\aspa \to [0, 1]$ and budget levels $\alpha_k \in (0, 1)$ for $k\in[K]$ as follows. 
We assume that the action $0$ incurs zero cost of any type, i.e., $c_k(s,0) = 0$ for all $s\in\sspa$ and $k\in[K]$. 
A policy $\pi$ for the $N$-armed system is \emph{feasible} if it satisfies, for all $t\geq 0$ and all $k\in[K]$,
\begin{equation}
    \label{eq:wcmdp:budget}
    \frac{1}{N}\sumN c_k(S_t^\pi(i), A_t^\pi(i)) \leq \alpha_k,
\end{equation}
where $S_t^\pi(i), A_t^\pi(i)$ are the state and action of arm $i$ at time $t$ under the policy $\pi$. 

For a fixed initial state vector $\veS_0$, we define the limsup and liminf average rewards per arm by
\begin{align*}
    & \rlimsup = \limsup_{T\to\infty}\frac{1}{NT}\sum_{t=0}^{T-1}\sumN\E{r(S_t^\pi(i),A_t^\pi(i))}, \\
    & \rliminf = \liminf_{T\to\infty}\frac{1}{NT}\sum_{t=0}^{T-1}\sumN\E{r(S_t^\pi(i),A_t^\pi(i))}.
\end{align*}
When these agree, their common value is the long-run average reward $\rsysn$.
Our objective is to maximize $\rliminf$ over feasible policies; its optimal value is $\ropt$.
Stationary policies on a finite augmented state space have a well-defined long-run average reward, and finite-state MDPs admit an optimal stationary policy \cite[Chapter~8 and Theorem~9.1.8]{Put_05}. For policies with a well-defined reward, the optimality gap is $\ropt-\rsysn$.

\paragraph{Scaled counts and notation.}
For $D\subseteq[N]$, we let $m(D)=|D|/N$ and
\[
\begin{aligned}
X_t^\pi(D,s) &= \frac{1}{N}\sum_{i\in D}\indibrac{S_t^\pi(i)=s}, \\
    Y_t^\pi(D,s,a) &= \frac{1}{N}\sum_{i\in D}\indibrac{S_t^\pi(i)=s,\ A_t^\pi(i)=a}.
\end{aligned}
\]
We write $X_t^\pi(D)$ and $Y_t^\pi(D)$ for the corresponding row vectors, and $Y_t^\pi=Y_t^\pi([N])$. The set function $X_t^\pi$ contains the same information as $\veS_t^\pi$. For nonempty $D$, $X_t^\pi(D)/m(D)$ is its empirical state distribution. Counts for an empty subset are zero. We write $\simplex(\sspa)$ for the probability simplex, $I_k$ for the $k$-by-$k$ identity matrix, and $\vone$ for the all-one row vector indexed by states. Distributions are row vectors, and policy superscripts are omitted when clear.

\subsection{LP relaxation and assumptions}\label{sec:wcmdp:lp}
We relax the per-step budget constraints \eqref{eq:wcmdp:budget} to their time-average counterparts, obtaining the following LP relaxation:
\begin{align}
    \label{eq:wcmdp-lp} \tag{LP-W}
    & \underset{\{y(s,a)\}_{s\in\sspa,a\in\aspa}}{\text{maximize}} \;\; \sumsa r(s,a)\, y(s,a) \\
    & \text{subject to} \;\; \sumsa c_k(s,a)\, y(s,a) \leq \alpha_k, \quad \forall k\in[K], \label{eq:wcmdp:budget-lp}\\
    & \sum_{s'\in\sspa,\, a\in\aspa} y(s', a) P(s', a, s) = \sum_{a\in\aspa} y(s, a),\forall s\in\sspa, \label{eq:wcmdp:flow}\\
    & \sum_{s'\in\sspa,a'\in\aspa} y(s',a') = 1, \quad y(s,a) \geq 0, \;\; \forall (s,a). \label{eq:wcmdp:simplex}
\end{align}
We let $\rrel$ denote the optimal value of the LP relaxation. A standard argument shows that $\rrel \geq \ropt$ \cite{ZhaHonWan_25_het}.
Fixing an optimal solution of \eqref{eq:wcmdp-lp}, $y^*$, we define $\statdist(s) \triangleq \sum_{a\in\aspa} y^*(s,a)$ as the induced \emph{optimal stationary distribution} over $\sspa$. 
We let $\sempty \triangleq \{s\in\sspa\colon \statdist(s) = 0\}$ be the set of transient states under $\pibs$.
We define the optimal single-armed policy associated with $y^*$ by
\begin{equation}\label{eq:wcmdp:single-arm-opt-def}
 \pibs(a\mid s)=\begin{cases}y^*(s,a)/\statdist(s),&\statdist(s)>0,\\ 1/|\aspa|,&s\in\sempty,\end{cases}
\end{equation}
with transition matrix $P_\pibs(s,s')=\sum_{a\in\aspa}\pibs(a\mid s)P(s,a,s')$.

\begin{assumption}[Aperiodic unichain]\label{assump:wcmdp:aperiodic-unichain}
    The transition matrix $P_{\pibs}$ induced by $\pibs$ has a simple eigenvalue $1$; all other eigenvalues have modulus strictly less than $1$; i.e., $P_{\pibs}$ defines an aperiodic unichain on $\sspa$.
\end{assumption}

\Cref{assump:wcmdp:aperiodic-unichain} depends only on the single-armed MDP. The two remaining assumptions, non-degeneracy (\Cref{assump:wcmdp:non-degeneracy}) and local stability (\Cref{assump:wcmdp:local-stability}), are stated in \Cref{sec:wcmdp:OLC}, where they arise naturally alongside the construction of the subroutine \lpfp. 
When there are multiple optimal solutions for \eqref{eq:wcmdp-lp}, we choose a fixed one that satisfies all assumptions.

\section{Two-set policy}\label{sec:wcmdp:policy}

The two-set policy combines two subroutines with complementary roles:
\frcontrol steers a subset's empirical state distribution toward $\statdist$, while \lpfp provides near-optimal control for a subset whose empirical state distribution is already close to $\statdist$.
By combining these subroutines, the policy seeks to enlarge the subset following \lpfp.
We first describe the two subroutines and then explain how the policy selects their respective subsets.

\begin{algorithmfigure}[t]
\begin{algorithmframe}
\textbf{Input}: A subset $D$ of arms with state counts\\
\hspace*{1em}$(|D|\,x(s))_{s\in\sspa}$, optimal single-armed policy $\pibs$
\begin{algorithmic}[1]
    \For{$s\in\sspa$}
        \State Independently assign each arm in state $s$
        \LineCont an action $a\sim\pibs(\cdot\mid s)$.
    \EndFor
\end{algorithmic}
\end{algorithmframe}
\caption{\frcontrol}
\label{alg:wcmdp:uoc}
\end{algorithmfigure}

\subsection{\frcontrol}
\label{sec:wcmdp:uoc}
The subroutine \emph{\frcontrol} (\Cref{alg:wcmdp:uoc}) assigns each arm an action by independent sampling from $\pibs(\cdot\mid s)$.

By \Cref{assump:wcmdp:aperiodic-unichain}, if all $N$ arms follow \Cref{alg:wcmdp:uoc}, their expected empirical state distribution evolves according to the transition matrix $P_\pibs$ and converges to $\statdist$.
In the two-set policy, this subroutine is applied to a \emph{subset of arms} for this purpose.

\subsection{\lpfp}\label{sec:wcmdp:OLC}

Next, we define the second subroutine, \lpfp.
It is defined to be a locally linear and near-optimal control when the system's empirical state distribution $x([N])$ is close to $\statdist$.
To construct this control, we compute a target state-action frequency $y$ as the solution of a linear system and use it to guide the actions in the next time step.
Because the construction of the subroutine relies on the optimal solution to the LP relaxation, $y^*$, which can be interpreted as an optimal fixed-point state-action frequency under the mean transitions, we refer to this subroutine as \emph{LP fixed-point}. 
\lpfp is inspired by the ``LP-update'' policy in the finite-horizon WCMDP literature \cite{GasGauYan_24_reopt}.

\Cref{sec:wcmdp:olc:assumptions} states the non-degeneracy and local stability assumptions (\Cref{assump:wcmdp:non-degeneracy,assump:wcmdp:local-stability}). \Cref{sec:wcmdp:olc:algorithm} gives the subroutine, \Cref{sec:wcmdp:olc:dynamics} its dynamics, and \Cref{sec:wcmdp:olc:feasibility} its feasibility conditions.

\subsubsection{Two assumptions: non-degeneracy and local stability}\label{sec:wcmdp:olc:assumptions}

We introduce two assumptions under which \lpfp is well-defined and induces locally stable dynamics.

For the fixed optimal LP solution $y^*$, let $\sempty$ be the set of transient states, $\Kset$ the set of tight budget constraints, and $\Uset$ the set of pairs $(s,a)$ with $s\notin\sempty$ and $y^*(s,a)=0$:
\[
\begin{aligned}
    \sempty &= \bigl\{s\in\sspa\colon \textstyle\sum_{a\in\aspa} y^*(s,a) = 0\bigr\}, \\
    \Kset &\triangleq \bigl\{k\in[K]\colon \textstyle\sumsa c_k(s,a)\,y^*(s,a) = \alpha_k\bigr\}, \\
    \Uset &\triangleq \bigl\{(s,a)\in\sspa\times\aspa\colon s\notin\sempty,\; y^*(s,a) = 0\bigr\}.
\end{aligned}
\]

For the rest of this subsection, we apply \lpfp to all $N$ arms, writing $x=X_t([N])\in\Delta(\sspa)$ for their empirical state distribution and $y\in\Delta(\sspa\times\aspa)$ for a state-action distribution. \Cref{sec:wcmdp:two-set} extends the subroutine to arbitrary subsets of arms.

We consider the following linear system in $y\in\R^{|\sspa|\times|\aspa|}$, parameterized by the empirical state distribution $x$:
\begin{equation}\label{eq:wcmdp:active-system}
    \begin{aligned}
        & y(s,a) = 0, & & (s,a)\in \Uset, \\
        & y(s,0) = y(s,1) = \cdots = y(s, |\aspa|-1), & & s\in\sempty, \\
        & \sumsa c_k(s,a) y(s,a) = \alpha_k, & & k\in\Kset, \\
        & \sum_{a\in\aspa} y(s,a) = x(s), & & s\in\sspa.
    \end{aligned}
\end{equation}
A solution $y$ to \eqref{eq:wcmdp:active-system} gives the target state-action frequency for \lpfp at this time step.
On a high level, the purpose of these constraints is to find a budget-feasible state-action frequency that recovers $y^*$ when $x = \statdist$.
Specifically, 
the first row of equations requires $y$ to share the support of $y^*$. 
The second row fixes the action distribution to be uniform conditional on any state in $\sempty$.
The third row requires the same budgets to be tight under $y$ and $y^*$.
The fourth row enforces consistency with the instantaneous state count $x$.

The equations in \eqref{eq:wcmdp:active-system} are fundamentally different from the constraints of the LP relaxation \eqref{eq:wcmdp-lp}. In particular, the equations here require $y$ to be an \emph{instantaneous} state-action frequency consistent with the current state distribution $x$, whereas $y^*$ is a stationary state-action frequency satisfying the flow-balance equation \eqref{eq:wcmdp:flow}. 
The first and second rows of \eqref{eq:wcmdp:active-system} also have no counterparts in \eqref{eq:wcmdp-lp}. 
Despite these differences, \eqref{eq:wcmdp:active-system} recovers $y^*$ when $x = \statdist$.

We write this linear system in vector form as
\begin{equation}
    \label{eq:wcmdp:active-system-vec}
    y \Cmat = \bigl(0,\;0,\;(\alpha_k)_{k\in\Kset},\;x\bigr),
\end{equation}
where $y$ and $x$ are regarded as row vectors, $\Cmat\in\mathbb{R}^{|\sspa||\aspa|\times d}$ denotes the coefficient matrix of this linear system, and $d = |\Uset| + (|\aspa|-1)|\sempty| + |\Kset| + |\sspa|$. A sufficient condition for \eqref{eq:wcmdp:active-system-vec} to have a solution $y$ for every $x\in\simplex(\sspa)$ is that $\Cmat$ has full column rank.

\begin{assumption}[Non-degeneracy]\label{assump:wcmdp:non-degeneracy}
    The matrix $\Cmat$ has full column rank.
\end{assumption}

A necessary condition for \Cref{assump:wcmdp:non-degeneracy} is $d \leq |\sspa||\aspa|$: $\Cmat$ must be square or tall.
Under \Cref{assump:wcmdp:aperiodic-unichain}, this dimensional condition holds when $|\sempty|=0$ and $y^*$ is a non-degenerate basic feasible solution of \eqref{eq:wcmdp-lp} (see \fullversionref[prop:wcmdp:cmat-dim]{subsec:wcmdp:non-degeneracy-validity}).

\Cref{assump:wcmdp:non-degeneracy} can be viewed as a strengthened version of $d \leq |\sspa||\aspa|$, as it additionally requires the columns of $\Cmat$ to be linearly independent.

\begin{remark}[Restless-bandit special case]
To gain intuition for \eqref{eq:wcmdp:active-system} and \Cref{assump:wcmdp:non-degeneracy}, we consider the RB specialization $|\aspa|=2$, $\Kset=\{1\}$, $c_1(s,a)=a$, with $\sempty=\emptyset$. The state space partition as $\sspa = S^+ \cup S^- \cup S^0$, where $S^+ = \{s\colon y^*(s,1)>0,\,y^*(s,0)=0\}$, $S^- = \{s\colon y^*(s,0)>0,\,y^*(s,1)=0\}$, and the neutral states $S^0 = \{s\colon y^*(s,0)>0,\,y^*(s,1)>0\}$. Then $\Uset = \{(s,0)\colon s\in S^+\}\cup\{(s,1)\colon s\in S^-\}$, and the linear system \eqref{eq:wcmdp:active-system} specializes into
\[
\begin{aligned}
y(s,0)&=0\;(s\in S^+), & y(s,1)&=0\;(s\in S^-), \\
\textstyle\sum_s y(s,1)&=\alpha, & y(s,0)+y(s,1)&=x(s)\;(s\in\sspa).
\end{aligned}
\]
The first two set of equtions pin down $y(s,1) = x(s)$ for $s\in S^+$ and $y(s,0) = x(s)$ for $s\in S^-$. To find $y(s,a)$ for $s\in S^0$, we discuss based on the number of neutral states. 
\begin{itemize}
    \item When there is exactly one neutral state $\sneu$ (the usual non-degenerate RB case), we can use the budget constraint to obtain $y(\sneu,1) = \alpha - \sum_{s\in S^+} x(s)$ and use the marginal constraint to obtain $y(\sneu, 0) = x(\sneu) - y(\sneu, 1)$. In this case, the solution $y$ is uniquely determined for any $x$, implying that $\Cmat$ is an invertible matrix and thus has full column rank (i.e., \Cref{assump:wcmdp:non-degeneracy} holds). This is consistent with the dimensionality of $\Cmat$: one can verify that $d = 2|\sspa| = |\sspa||\aspa|$, so $\Cmat$ is a square matrix.
    \item When there is no neutral state, the linear system is overdetermined with $|S^+| + |S^-| + 1 + |\sspa| = 2|\sspa|+1 > |\sspa||\aspa|$ equations, so it does not have solutions for all $x$; $\Cmat$ is wide and cannot have full column rank.
    \item When there is more than one neutral state for RBs, the solution $y$ always exists but is not unique, so $\Cmat$ still has full column rank and \Cref{assump:wcmdp:non-degeneracy} still holds. In this case, we will select a fixed solution $y$, as we discuss immediately below while setting up for the next assumption.
\end{itemize}
When there is exactly one neutral state, the affine target allocates action $1$ to the states in $S^+$ and uses the neutral state to meet the budget constraint.

\end{remark}

Next, we prepare to state the local stability assumption.  
To this end, we need to fix a solution of \eqref{eq:wcmdp:active-system-vec}. First, we note that at $x=\statdist$, one such solution is $y^*$ by definition, giving
\begin{equation}\label{eq:wcmdp:ystar-cmat}
    y^*\Cmat = \bigl(0,\;0,\;(\alpha_k)_{k\in\Kset},\;\statdist\bigr).
\end{equation}
For general $x$, we need to invert the matrix $\Cmat$. 
Since $\Cmat$ has full column rank, it admits a left inverse $\Cinv\in\mathbb{R}^{d\times|\sspa||\aspa|}$ satisfying $\Cinv\,\Cmat = I_d$.\footnote{When $d < |\sspa||\aspa|$, the left inverse is not unique. In this case, we fix a left inverse such that \Cref{assump:wcmdp:local-stability} holds.} 
We denote by $\Cinvs\in\mathbb{R}^{|\sspa|\times|\sspa||\aspa|}$ the $|\sspa|$ rows of $\Cinv$ corresponding to the marginal-distribution block (the last row of equations in \eqref{eq:wcmdp:active-system}); then
\begin{equation}\label{eq:wcmdp:Cinvs-Cmat}
    \Cinvs\,\Cmat = \bigl(0,\;0,\;0,\;I_{|\sspa|}\bigr),
\end{equation}
Therefore, the following $y$ is a solution of \eqref{eq:wcmdp:active-system-vec}:
\begin{equation}\label{eq:wcmdp:y-linear}
    y \triangleq  y^* + (x - \statdist)\,\Cinvs.
\end{equation}
To verify this, we right-multiply \eqref{eq:wcmdp:y-linear} by $\Cmat$ and apply \eqref{eq:wcmdp:ystar-cmat} together with \eqref{eq:wcmdp:Cinvs-Cmat}, obtaining
\begin{equation*}
\begin{aligned}
y\Cmat &= y^*\Cmat + (x-\statdist)\,\Cinvs\,\Cmat \\
    &= \bigl(0,\,0,\,(\alpha_k)_{k\in\Kset},\,\statdist\bigr) + \bigl(0,\,0,\,0,\,x-\statdist\bigr) \\
    &= \bigl(0,\,0,\,(\alpha_k)_{k\in\Kset},\,x\bigr),
\end{aligned}
\end{equation*}
as required.
We will select this particular $y$ as the solution of \eqref{eq:wcmdp:active-system-vec} in the rest of the subsection and use it to construct the \lpfp subroutine.

Our next assumption concerns the dynamics induced by the state-action frequency $y$ in \eqref{eq:wcmdp:y-linear}. We let $\mathcal{P}\in\mathbb{R}^{(|\sspa||\aspa|)\times|\sspa|}$ be the transition matrix with entries $\mathcal{P}_{(s,a),s'} = P(s,a,s')$, and define
\begin{equation}\label{eq:wcmdp:phi-def}
    \Phi \triangleq  \Cinvs \mathcal{P} - \vone^\top\statdist\,\Cinvs\mathcal{P},
\end{equation}
where $\vone^\top\in\mathbb{R}^{|\sspa|}$ is the all-ones column vector. We will show in \Cref{sec:wcmdp:olc:dynamics} that $\Phi$ captures the one-step transition of $X_t([N]) - \statdist$ under the subroutine. 
We assume the following:

\begin{assumption}[Local stability]\label{assump:wcmdp:local-stability}
    The spectral radius of $\Phi$ defined in \eqref{eq:wcmdp:phi-def} is strictly less than $1$.
\end{assumption}

Intuitively, local stability gives \lpfp a restoring effect near $\statdist$. Under the locally linear mean dynamics, an initial deviation $v$ from $\statdist$ evolves as $v\Phi^t$; the spectral-radius condition ensures that these deviations decay geometrically over time.

The purpose of centering term $-\vone^\top\statdist\,\Cinvs\mathcal{P}$ in $\Phi$ is to replace the trivial eigenvalue $1$ of $\Cinvs\mathcal{P}$ with $0$ while preserving the effect of the operator on the probability simplex, as formalized in \fullversionref[prop:wcmdp:phi-vs-CinvsP]{subsec:wcmdp:local-stability-validity}.

\subsubsection{The subroutine}\label{sec:wcmdp:olc:algorithm}

\begin{algorithmfigure}[t]
\begin{algorithmframe}
\textbf{Input}: A set of $n$ arms with number of arms\\
\hspace*{1em}in each state $(z(s))_{s\in\sspa}$, LP solution $y^*$,\\
\hspace*{1em}left-inverse submatrix $\Cinvs$, feasibility radius $\eta$\\
\textbf{If} $n=0$, \textbf{return} without assigning actions.\\
\textbf{Assert} \Cref{assump:wcmdp:non-degeneracy} and $\norm{z/n - \statdist}_\umat \leq \eta$
\begin{algorithmic}[1]
    \State Compute $y \gets y^* + (z/n - \statdist)\,\Cinvs$ \label{alg:wcmdp:olc:compute}
    \For{each $s\in\sspa$ and each $a\in\aspa\setminus\{0\}$}
        \State Uniformly choose $\lfloor n\,y(s,a)\rfloor$ unassigned \label{alg:wcmdp:olc:assign}
        \LineCont arms in state $s$ and assign them action $a$
    \EndFor
    \State Assign action $0$ to all remaining arms
\end{algorithmic}
\end{algorithmframe}
\caption{\lpfp}
\label{alg:wcmdp:olc}
\end{algorithmfigure}

The \lpfp subroutine is formally stated as the pseudocode in \Cref{alg:wcmdp:olc}: for a nonempty input, it computes $y$ from \eqref{eq:wcmdp:y-linear} and then assigns arms by floor-rounding the target counts $n\,y(s,a)$. Within each state, the resulting allocation is uniform over all assignments realizing those counts. The threshold $\eta$ and the weighted norm $\norm{\cdot}_\umat$ that appear in the assertion are derived in \Cref{sec:wcmdp:olc:feasibility}.

\subsubsection{Linear dynamics under \lpfp}\label{sec:wcmdp:olc:dynamics}

\Cref{alg:wcmdp:olc} induces a linear transition dynamics in expectation, up to an $O(1/N)$ rounding error. The following lemma makes this precise. %

\begin{restatable}[\lpfp transition dynamics]{lemma}{wcmdpolctransition}\label{lem:wcmdp:olc-transition}
    Under \Cref{assump:wcmdp:non-degeneracy,assump:wcmdp:local-stability}, suppose $\norm{X_t([N]) - \statdist}_\umat \leq \eta$ at time $t$ and all arms follow \Cref{alg:wcmdp:olc}. Let $Y_t$ be the realized state-action distribution, $y_t = y^* + (X_t([N]) - \statdist)\,\Cinvs$ as in \eqref{eq:wcmdp:y-linear}, and $\mathcal{P}$ and $\Phi$ be as in \eqref{eq:wcmdp:phi-def}. Then
    \begin{equation}\label{eq:wcmdp:olc-transition}
\begin{aligned}
& \E{X_{t+1}([N]) - \statdist \givenplain X_t} \\
    & \quad = (X_t([N]) - \statdist)\,\Phi \;+\; (Y_t - y_t)\,\mathcal{P},
\end{aligned}
\end{equation}
    where $\|(Y_t - y_t)\,\mathcal{P}\|_1 \leq \norm{Y_t - y_t}_1 \leq 2|\sspa|(|\aspa|-1)/N$.
\end{restatable}

The proof of \Cref{lem:wcmdp:olc-transition} is given in \fullversionref{app:wcmdp:subroutine-transition-lemmas}. %

\subsubsection{When is the subroutine feasible?}\label{sec:wcmdp:olc:feasibility}

Next, we derive the assertion on \Cref{alg:wcmdp:olc}, the sufficient condition for the \lpfp subroutine to be applicable.
Specifically, we find some weighted norm $\norm{\cdot}_\umat$ and a \emph{feasibility radius} $\eta > 0$, such that whenever $\norm{z/n - \statdist}_\umat \leq \eta$, we have
\begin{enumerate}[leftmargin=3em, label=(\roman*)]
    \item \emph{Non-negativity:} $y(s,a) \geq 0$ for all $(s,a)$;
    \item \emph{Budget constraints:} $\sumsa c_k(s,a)\, y(s,a) \leq \alpha_k$ for all $k\in[K]$.
\end{enumerate}
which guarantees that $y$ is a valid probability distribution and that the resulting actions of \lpfp satisfy the budget constraints.
Neither (i) nor (ii) is encoded in the linear system \eqref{eq:wcmdp:y-linear}, and both hold only when $x$ lies sufficiently close to $\statdist$.

\begin{restatable}{definition}{wcmdpwudef}\label{def:wcmdp:w-and-u}
    Assume \Cref{assump:wcmdp:aperiodic-unichain,assump:wcmdp:non-degeneracy,assump:wcmdp:local-stability}. Let $\wmat$ and $\umat$ be the $|\sspa|$-by-$|\sspa|$ matrices given by
    \begin{align}
        \label{eq:wcmdp:w-def}
        \wmat &= \sum_{k=0}^\infty (P_\pibs - \vone^\top \statdist)^k ((P_\pibs - \vone^\top \statdist)^\top)^k, \\
        \label{eq:wcmdp:u-def}
        \umat &= \sum_{k=0}^\infty \Phi^k (\Phi^\top)^k,
    \end{align}
    where $\Phi$ is given by \eqref{eq:wcmdp:phi-def}.
\end{restatable}

We write $\lamw \triangleq \norm{\wmat}_2$ and $\lamu \triangleq \norm{\umat}_2$, and let $\norm{\cdot}_\wmat$ and $\norm{\cdot}_\umat$ denote the $\wmat$- and $\umat$-weighted $L_2$ norms on $\R^{|\sspa|}$, defined by $\norm{u}_\wmat = \sqrt{u\,\wmat\,u^\top}$ and $\norm{u}_\umat = \sqrt{u\,\umat\,u^\top}$ for any row vector $u$.

Under \Cref{assump:wcmdp:local-stability} ($\rho(\Phi)<1$), $\umat$ is well-defined and satisfies the contraction property in \fullversionref{app:wcmdp:weighted-l2-norm-lemmas}. The feasibility radius $\eta$ is then
\begin{equation}\label{eq:wcmdp:eta-def}
\begin{aligned}
\eta &\triangleq \frac{\epsilon}{\opnorm{\Cinvs}_{\umat^{-1}}}, \qquad \text{where} \\
    \opnorm{\Cinvs}_{\umat^{-1}} &\triangleq \max_{(s,a)}\norm{(\Cinvs\, e_{(s,a)})^\top}_{\umat^{-1}},
\end{aligned}
\end{equation}
$e_{(s,a)}\in \R^{|\sspa||\aspa|}$ is the one-hot column vector with $1$ at entry $(s,a)$, and the constant $\epsilon$ is given by
\begin{equation}\label{eq:wcmdp:eps-eta}
\begin{aligned}
\epsilon &\triangleq \min\!\Big\{ \min_{\substack{(s,a)\colon y^*(s,a)>0}} y^*(s,a), \\
    &\quad \min_{k\notin\Kset} \frac{\alpha_k - \sumsa c_k(s,a)\,y^*(s,a)}{|\sspa||\aspa|} \Big\}.
\end{aligned}
\end{equation}

The inner minimum over $k\notin\Kset$ is interpreted as $+\infty$ when all budgets are tight. The denominator in \eqref{eq:wcmdp:eta-def} is positive because \eqref{eq:wcmdp:Cinvs-Cmat} implies $\Cinvs\ne 0$.

We now argue that whenever $\norm{x - \statdist}_\umat \leq \eta$, the target $y = y^* + (x - \statdist)\,\Cinvs$ satisfies (i) and (ii) above. For every $(s,a)\in\sspa\times\aspa$, the Cauchy--Schwarz inequality gives
\begin{equation}\label{eq:wcmdp:y-bound}
\abs{y(s,a)-y^*(s,a)}\leq\opnorm{\Cinvs}_{\umat^{-1}}\norm{x-\statdist}_\umat\leq\epsilon.
\end{equation}
Since $y^*(s,a) \geq \epsilon$ for all $(s,a)\notin\Uset\cup(\sempty\times\aspa)$ (by definition of $\epsilon$), it follows that $y(s,a) \geq 0$, verifying (i). The budget constraints (ii) is verified similarly. The formal statement is given in \fullversionref{lem:wcmdp:linear-feasibility}.

The feasibility condition $\norm{x-\statdist}_\umat\leq\eta$ uses floor rounding for every nonzero action, with remaining arms assigned action $0$.

\subsection{Two-set policy and feasibility}\label{sec:wcmdp:two-set}
The two-set policy (\Cref{alg:wcmdp:two-set}) is defined as follows: at each time step, it selects two subsets, $\Db_t$ and $\Da_t$, letting them follow the two subroutines, \lpfp and \frcontrol, respectively. The rest of the arms can take arbitrary actions as long as the budget constraints are satisfied, and we simply assign them action $0$.\footnote{Unspecified choices in \Cref{alg:wcmdp:two-set} are made uniformly at random among the admissible alternatives, using fresh randomness at each step.}

To select the subset $\Db_t$, we define the \emph{slack function}
\begin{equation}\label{eq:wcmdp:slack-def}
    \slkb(X_t, D) \triangleq  \eta\,m(D) \;-\; \norm{X_t(D) - m(D)\statdist}_\umat,
\end{equation}
For $D\ne\emptyset$, this slack is non-negative if and only if $\norm{X_t(D)/m(D) - \statdist}_\umat \leq \eta$, i.e., the empirical state distribution of the arms in $D$ lies within the feasibility radius $\eta$ of $\statdist$. For $D=\emptyset$, the slack is zero and the subroutine performs no assignments. Consequently, whenever $\slkb(X_t, D) \geq 0$, \lpfp is applicable to the arms in $D$.
\begin{definition}[$\errtol$-maximal feasible set]
\label{def:wcmdp:maximal-feasible-set}
    Given the current system state $x$ and $\errtol\geq0$,
    a set of arms $D\subseteq[N]$ is $\errtol$-maximal feasible if the two conditions hold:
    (1) $\slk(x, D)\geq 0$ or $D=\emptyset$;
    (2) for any $D'$ such that $D \subseteq D'\subseteq [N]$ and $\slkb(x, D') \geq \errtol$, we have $m(D') \leq m(D) + \errtol$.
\end{definition}

We choose $\Db_t$ to be a $\errtol$-maximal feasible set in the sense of \Cref{def:wcmdp:maximal-feasible-set} for some predetermined $\errtol\geq0$ with $\errtol=O(1/N)$, with some additional requirements specified in Lines \ref{alg:wcmdp:two-set:db-begin}--\ref{alg:wcmdp:two-set:db-end} of \Cref{alg:wcmdp:two-set}.

The subset $\Da_t\subseteq[N]\setminus\Db_t$ is chosen with size $\lfloor\alpha_{\min}(N-|\Db_t|)\rfloor$ and the nesting condition on Line~\ref{alg:wcmdp:two-set:da}, where $\alpha_{\min} \triangleq \min_{k\in[K]} \alpha_k$ ensures that the arms in $\Da_t$ can follow \frcontrol while satisfying the budget constraints. Both set selections admit at least one admissible choice, as shown in \fullversionref[lem:wcmdp:subroutine-conform]{sec:wcmdp:subroutine-conform}.

\begin{algorithmfigure}[t]
\begin{algorithmframe}
\textbf{Input}: number of arms $N$, budgets $(\alpha_k N)_{k\in[K]}$,\\
\hspace*{1em}an optimal solution of LP-relaxation $y^*$,\\
\hspace*{1em}left-inverse submatrix $\Cinvs$, feasibility radius $\eta$,\\
\hspace*{1em}$\alpha_{\min} \triangleq \min_{k\in[K]} \alpha_k$,\\
\hspace*{1em}error tolerance $\errtol\geq0$ with $\errtol = O(1/N)$,\\
\hspace*{1em}initial system state $X_0$, initial state vector $\veS_0$,\\
\hspace*{1em}initial subsets $\Db_{-1} = \Da_{-1} = \emptyset$
\begin{algorithmic}[1]
    \For{$t = 0, 1, \ldots$}
        \If{$\slkb(X_t, [N]) \geq 0$} \label{alg:wcmdp:two-set:db-begin}
            \State Let $\Db_t = [N]$
        \ElsIf{$\slkb(X_t, \Db_{t-1}) \geq 0$}
            \State Let $\Db_t$ be any $\errtol$-maximal
            \LineCont feasible subset such that $\Db_t \supseteq \Db_{t-1}$
        \Else
            \State Let $\Db_t$ be any $\errtol$-maximal \label{alg:wcmdp:two-set:db-end}
            \LineCont feasible subset
        \EndIf
        \State Let $\Da_t \subseteq [N]\setminus\Db_t$ with \label{alg:wcmdp:two-set:da}
        \LineCont $|\Da_t| = \lfloor\alpha_{\min}(N-|\Db_t|)\rfloor$,
        \LineCont s.t.\ either $\Da_t \supseteq \Da_{t-1}\setminus\Db_t$
        \LineCont or $\Da_t \subseteq \Da_{t-1}\setminus\Db_t$
        \State Set $A_t(i)$ for $i\in\Db_t$ using \Cref{alg:wcmdp:olc} \label{alg:wcmdp:two-set:db-action}
        \State Set $A_t(i)$ for $i\in\Da_t$ using \Cref{alg:wcmdp:uoc} \label{alg:wcmdp:two-set:da-action}
        \State Set $A_t(i) = 0$ for all $i\notin\Db_t\cup\Da_t$ \label{alg:wcmdp:two-set:remaining-action}
        \State Apply $(A_t(i))_{i\in[N]}$; observe $\veS_{t+1}$
    \EndFor
\end{algorithmic}
\end{algorithmframe}
\caption{Two-Set Policy for WCMDPs}
\label{alg:wcmdp:two-set}
\end{algorithmfigure}

\subsubsection*{Feasibility under the global budget constraint.}
By construction, whenever $\slkb(X_t, \Db_t) \geq 0$, \Cref{alg:wcmdp:olc} applied to $\Db_t$ is well-defined, and the type-$k$ cost of arms in $\Db_t$ satisfies $\frac{1}{N}\sum_{i\in\Db_t} c_k(S_t(i), A_t(i)) \leq m(\Db_t)\,\alpha_k$ for each $k\in[K]$.
The \frcontrol on $\Da_t$ incurs at most $|\Da_t|/N \leq \alpha_{\min}(1 - m(\Db_t)) \leq \alpha_k(1 - m(\Db_t))$ type-$k$ cost for every $k$ (we recall that $c_k(s,a) \leq 1$ for all $s\in\sspa$ and $a\in\aspa$).
Since the remaining arms take the zero-cost action $0$, the full system satisfies 
\[
\begin{aligned}
& \frac{1}{N}\sum_{i=1}^N c_k(S_t(i), A_t(i)) \\
    & \quad \leq \alpha_k m(\Db_t) + \alpha_k\big(1-m(\Db_t)\big) = \alpha_k, \quad \forall k\in[K].
\end{aligned}
\] 
The full subroutine conformity argument appears in \fullversionref[lem:wcmdp:subroutine-conform]{sec:wcmdp:subroutine-conform}.

\section{Optimality gap}\label{sec:wcmdp:proof-outline}

\begin{theorem}[Optimality gap for the WCMDP two-set policy]\label{thm:wcmdp:achievability}
    Suppose \Cref{assump:wcmdp:aperiodic-unichain}, \Cref{assump:wcmdp:non-degeneracy}, and \Cref{assump:wcmdp:local-stability} hold. Let $\pi$ denote the two-set policy given in \Cref{alg:wcmdp:two-set} with error tolerance $\errtol=O(1/N)$. Then $\pi$ satisfies
    \begin{equation}\label{eq:wcmdp:main-result}
        \rrel - \rsysn = O\!\left(\frac{1}{N}\right).
    \end{equation}
\end{theorem}

The proof uses \Cref{lem:wcmdp:olc-transition,lem:wcmdp:inst-reward}, which characterize the mean transitions and instantaneous reward under \lpfp. Both include an $O(1/N)$ floor-rounding residual that contributes the $O(1/N)$ term in our bound. The proof outline appears in \Cref{sec:wcmdp:analysis}, with details in \fullversionref{app:wcmdp:upper}.

\section{Analysis of the two-set policy}\label{sec:wcmdp:analysis}

This section outlines the proof of \Cref{thm:wcmdp:achievability}; detailed arguments appear in \fullversionref{app:wcmdp:upper}.

Throughout the analysis we assume \Cref{assump:wcmdp:aperiodic-unichain,assump:wcmdp:non-degeneracy,assump:wcmdp:local-stability} and set $\threshbar\triangleq\eta$. If $\rmax=0$, the optimality gap is zero, so we assume $\rmax>0$ below.

Under the two-set policy, $\Fullstate_t\triangleq(X_t,\Db_t,\Da_t)$ is a time-homogeneous finite-state Markov chain. We let $\fullssp$ denote its state space and write $\fullstate=(x,\Db,\Da)$ for a generic element of $\fullssp$.
For ease of presentation, we suppose that $\Fullstate_t$ is aperiodic. For any fixed initial state $(X_0,\Db_0,\Da_0)$, we let $\Fullstate_\infty=(X_\infty,\Db_\infty,\Da_\infty)$ have its limiting distribution. For a general finite-state chain, we instead let $\Fullstate_\infty$ have the limit of the time-averaged state distributions for the chosen initial state. This distribution is stationary and satisfies the same long-run reward identity below, so the argument applies unchanged.

We define the one-step drift operator
\[
    \Delta f(\fullstate) \triangleq  \E{f(\Fullstate_{t+1}) \given \Fullstate_t = \fullstate} - f(\fullstate)
\]
for any function $f\colon\fullssp\to\R$. Stationarity gives $\E{\Delta f(\Fullstate_\infty)}=0$.

We define the expected instantaneous reward under the two-set policy as
\[
    r^\pi(\fullstate) \triangleq  \sumsa r(s,a)\,\Ebig{Y_t(s,a) \givenbig \Fullstate_t=\fullstate}.
\]
The long-run average reward of the policy then satisfies $\rsysn = \E{r^\pi(\Fullstate_\infty)}$.

\subsection{Proof outline}\label{sec:wcmdp:proof-outline-app}

\subsubsection{Lyapunov framework}\label{sec:wcmdp:proof-outline:framework}
To show $\rrel - \rsysn \leq \loworder$ for some low-order term $\loworder \geq 0$,
it suffices to find $V, \drftup\colon\fullssp\to\R$ satisfying the drift condition $\Delta V(\fullstate) \leq -\drftup(\fullstate) + \loworder$ and the dominance condition $\drftup(\fullstate) \geq \rrel - r^\pi(\fullstate)$ for all $\fullstate\in\fullssp$.

\subsubsection{Proving \texorpdfstring{$O(1/\sqrt{N})$}{O(1/sqrt(N))} optimality gap}\label{sec:wcmdp:proof-outline:warmup}
As a warm-up, we bound the reward gap by the distance from $\statdist$, allowing for rounding:
\begin{equation}\label{eq:wcmdp:warmup-reward-bound}
\begin{aligned}
\rrel-r^\pi(\fullstate)
    &\leq K_0\norm{\statdist-x([N])}_\umat \\
    &\quad + \frac{2\rmax|\sspa|(|\aspa|-1)}{N}.
\end{aligned}
\end{equation}
Here $K_0>0$ is a constant independent of $N$.
We control this distance using two \emph{subset Lyapunov functions}:
\begin{align}
    \label{eq:hw-def}
    \hw(x, D) &= \norm{x(D) - m(D)\statdist}_\wmat \\
    \label{eq:hu-def}
    \hu(x, D) &= \norm{x(D) - m(D)\statdist}_\umat,
\end{align}
where matrices $\wmat, \umat$ are defined via $P_\pibs, \Phi$ as in \Cref{def:wcmdp:w-and-u}.

We then set
\begin{equation}\label{eq:wcmdp:sketch:Vds}
\Vds(\fullstate)\triangleq\hu(x,\Db)+\hw(x,\Da)+\Lds(1-m(\Db)),
\end{equation}
where $\Lds = 2\lamu^{1/2} + 4\lamw^{1/2} - 2\lamw^{1/2}\beta$ with $\beta \triangleq \alpha_{\min}$ and $\lamw, \lamu$ being the spectral radii of $\wmat$ and $\umat$. %

\begin{restatable}[Properties of $\Vds$]{lemma}{wcmdpvdsproperties}\label{lem:wcmdp:Vds-properties}
    Assume \Cref{assump:wcmdp:aperiodic-unichain,assump:wcmdp:non-degeneracy,assump:wcmdp:local-stability} hold, and let $\threshbar = \eta$. There exist $\rhoFinal\in(0,1)$ and $\KVds,\KVdst,\CVds > 0$ independent of $N$ such that, for all $t\geq 0$ and all $\fullstate=(x,\Db,\Da)\in\fullssp$, we have
    \begin{align}
    & \E{\bigl(\Vds(\Fullstate_{t+1}) - \rhoFinal\,\Vds(\Fullstate_t)\bigr)^{\!+} \given \Fullstate_t=\fullstate} \leq \frac{\KVds}{\sqrt{N}}, \label{eq:wcmdp:Vds-mean} \\
    & \E{\Bigl(\Vds(\Fullstate_{t+1}) - \rhoFinal\,\Vds(\Fullstate_t) - \tfrac{(1-\rhoFinal)\threshbar}{2}\Bigr)^{\!+} \given \Fullstate_t=\fullstate} \nonumber \\
    & \quad \leq \KVdst\exp(-\CVds N), \label{eq:wcmdp:Vds-tail}
\end{align}
    and
    \begin{equation}\label{eq:wcmdp:Vds-dom}
        \Vds(\fullstate) \geq  \norm{x([N]) - \statdist}_\umat.
    \end{equation}
\end{restatable}

The proof of \Cref{lem:wcmdp:Vds-properties}, including the transition residual from \Cref{lem:wcmdp:olc-transition}, is given in \fullversionref{sec:wcmdp:pf-Vds-properties}.

Combining \eqref{eq:wcmdp:Vds-mean} with \eqref{eq:wcmdp:Vds-dom} gives $\Delta\Vds(\fullstate) \leq -(1-\rhoFinal)\norm{x([N])-\statdist}_\umat + \KVds/\sqrt{N}$. Taking stationary expectations and applying \eqref{eq:wcmdp:warmup-reward-bound} yields
\[
\begin{aligned}
\rrel-\rsysn &\leq \frac{K_0\KVds}{(1-\rhoFinal)\sqrt{N}} + \frac{2\rmax|\sspa|(|\aspa|-1)}{N}\\
    &= O(1/\sqrt{N}).
\end{aligned}
\]

\subsubsection{Proving \texorpdfstring{$O(1/N)$}{O(1/N)} optimality gap}\label{sec:wcmdp:proof-outline:exponential}
To sharpen the optimality gap, we use a second Lyapunov function:
\begin{equation}\label{eq:wcmdp:sketch:Vll}
    \Vll(\fullstate) \triangleq  \Bigl(\Vds(\fullstate) - \tfrac{\threshbar}{2}\Bigr)^{\!+} + \Lvll\,\locapprox(\fullstate),
\end{equation}
with $\locapprox(\fullstate) \triangleq (\statdist - x([N]))\,Q\,\gvec$, $Q \triangleq (I-\Phi)^{-1}$ (well-defined by \Cref{assump:wcmdp:local-stability}), $\lamq$ as in \eqref{eq:lamq-ef}, and $\Lvll \triangleq (1-\rhoFinal)\threshbar/(4\lamq\Kg + 4\rmax + 4\Kg)$, where the vector $\gvec$ and the constant $\Kg$ are defined in \Cref{lem:wcmdp:inst-reward} below. Here
\begin{equation}\label{eq:lamq-ef}
\lamq\triangleq\max_{v\in\simplex(\sspa)}\max\{\norm{(v-\statdist)Q}_1,\norm{(v-\statdist)\Phi Q}_1\}.
\end{equation}

Near $\statdist$, the local-approximation term $\locapprox$ produces a drift equal to the negative instantaneous reward gap up to an $O(1/N)$ rounding residual.
The truncated $\Vds$ term offsets the additional error outside this local region.
The following three lemmas establish the resulting drift bound for $\Vll$.

The first lemma characterizes the instantaneous expected reward near $\statdist$. Its proof requires new arguments specific to \lpfp; see \fullversionref{sec:wcmdp:pf-inst-reward}.

\begin{restatable}[Instantaneous reward under \lpfp]{lemma}{wcmdpinstrew}\label{lem:wcmdp:inst-reward}
    Under the WCMDP two-set policy, for any $\fullstate=(x,\Db,\Da)\in\fullssp$ with $\norm{x([N]) - \statdist}_\umat \leq \threshbar \triangleq \eta$, we have $\Db=[N]$, and
    \[
        r^\pi(\fullstate) = \rhat(x([N])) \;+\; \errrew(\fullstate),
    \]
    where $\rhat(v) \triangleq \rrel + (v-\statdist)\,\gvec$ with $\gvec \triangleq \Cinvs\,\rvec^\top - (\statdist\,\Cinvs\,\rvec^\top)\,\vone^\top \in \R^{|\sspa|}$, $\rvec$ is the row vector $(r(s,a))_{s\in\sspa,a\in\aspa}$, and $\errrew \colon \fullssp\to\R$ satisfies $|\errrew(\fullstate)| \leq 2\rmax\,|\sspa|(|\aspa|-1)/N$. Moreover, $\norm{\gvec}_\infty \leq \Kg \triangleq 2\rmax\bigl(1 + \sqrt{2}\,\lamu^{1/2}/\eta\bigr)$.
\end{restatable}

The next two lemmas control the truncated and local-approximation terms; their proofs appear in \fullversionref{sec:wcmdp:pf-drift-truncate,sec:wcmdp:pf-drift-local}.

\begin{restatable}[Drift of truncated $\Vds$]{lemma}{wcmdpdrifttruncate}\label{lem:wcmdp:drift-Vds-truncate-term}
    For any $\fullstate,\fullstate'\in\fullssp$,
    \begin{align*}
    & \bigl(\Vds(\fullstate') - \tfrac{\threshbar}{2}\bigr)^{\!+} - \bigl(\Vds(\fullstate) - \tfrac{\threshbar}{2}\bigr)^{\!+} \\
    & \quad \leq -\frac{(1-\rhoFinal)\threshbar}{2}\indibrac{\Vds(\fullstate) > \threshbar} \\
    & \qquad + \Bigl(\Vds(\fullstate') - \rhoFinal\Vds(\fullstate) - \tfrac{(1-\rhoFinal)\threshbar}{2}\Bigr)^{\!+}.
\end{align*}
\end{restatable}

\begin{restatable}[Drift of $\locapprox$]{lemma}{wcmdpdriftlocal}\label{lem:wcmdp:drift-local-approx-term}
    For any $\fullstate=(x,\Db,\Da)\in\fullssp$,
    \begin{equation}\label{eq:wcmdp:drift-locapprox}
\begin{aligned}
& \Delta\locapprox(\fullstate) \\
    & \leq -(\rrel-r^\pi(\fullstate)) + \frac{\Kla}{N} \\
    & \quad + (2\lamq\Kg+2\rmax+2\Kg)\indibrac{\norm{x([N])-\statdist}_\umat>\threshbar},
\end{aligned}
\end{equation}
    where $\Kla \geq 0$ is a constant independent of $N$ and $\fullstate$.
\end{restatable}

\subsubsection{Proof of \hcref{thm:wcmdp:achievability}}\label{sec:wcmdp:proof-outline:proof}
Combining \Cref{lem:wcmdp:drift-Vds-truncate-term} with \eqref{eq:wcmdp:Vds-tail} of \Cref{lem:wcmdp:Vds-properties} and adding $\Lvll$ times \eqref{eq:wcmdp:drift-locapprox}, we get
\[
\Delta\Vll(\fullstate)\leq-\Lvll(\rrel-r^\pi(\fullstate))+\frac{\Lvll\Kla}{N}+\KVdst\exp(-\CVds N).
\]
Taking the expectation over $\Fullstate_\infty$ and using $\E{\Delta\Vll(\Fullstate_\infty)}=0$ gives $\rrel - \rsysn \leq \Kla/N + (\KVdst/\Lvll)\exp(-\CVds N) = O(1/N)$. The dominant term $\Kla/N$ comes from \Cref{lem:wcmdp:drift-local-approx-term}.

\bibliographystyle{IEEEtran}
\bibliography{bib/refs}

\iffullversion
\clearpage
\newpage
\appendices
\section{Auxiliary facts about the assumptions}\label{app:wcmdp:assumptions}

This appendix shows some auxiliary facts about \Cref{assump:wcmdp:non-degeneracy,assump:wcmdp:local-stability}.
In \Cref{subsec:wcmdp:non-degeneracy-validity},
we compute the dimension of $\Cmat$ under the additional condition $|\sempty|=0$, establishing a sufficient condition for $d \leq |\sspa||\aspa|$. This dimensional condition is necessary for \Cref{assump:wcmdp:non-degeneracy}. 
Then, in \Cref{subsec:wcmdp:local-stability-validity},
we investigate properties of the matrix $\Phi$ (defined in \eqref{eq:wcmdp:phi-def}) used in the local stability assumption (\Cref{assump:wcmdp:local-stability}).
These properties explain the centering in the definition of $\Phi$ and the role of the local stability assumption.

We use the notation $\rho(M)$ to refer to the spectral radius of a square matrix $M$. We continue using $\vone\in\R^{|\sspa|}$ to denote the all-one row vector indexed by states, consistent with other parts of the paper. We write $\vone_{|\sspa||\aspa|}\in\R^{|\sspa||\aspa|}$ with explicit subscript to denote the all-one row vector indexed by state-action pairs.

\subsection{Dimension of \texorpdfstring{$\Cmat$}{C*} and validity of \hcref{assump:wcmdp:non-degeneracy}}\label{subsec:wcmdp:non-degeneracy-validity}

\Cref{assump:wcmdp:non-degeneracy} requires $\Cmat\in\R^{|\sspa||\aspa|\times d}$ to have full column rank, which requires $d$ to be no larger than $|\sspa||\aspa|$. Under \Cref{assump:wcmdp:aperiodic-unichain}, the proposition below establishes this dimensional condition when $|\sempty|=0$ (no transient states) and $y^*$ is a non-degenerate basic feasible solution (BFS) of \eqref{eq:wcmdp-lp}. We account for redundant equality constraints by calling a BFS $y^*$ non-degenerate if the \emph{rank} of the equality-constraint coefficient matrix plus the \emph{number} of tight inequality constraints at $y^*$ equals the number of variables, $\abs{\sspa}\abs{\aspa}$.

\begin{proposition}[Dimension of $\Cmat$]\label{prop:wcmdp:cmat-dim}
    Suppose \Cref{assump:wcmdp:aperiodic-unichain} holds, $|\sempty|=0$, and $y^*$ is a non-degenerate basic feasible solution of \eqref{eq:wcmdp-lp}. Then $d = |\sspa||\aspa|$, so $\Cmat$ is square.
\end{proposition}

\begin{proof}
We write the flow-balance equations \eqref{eq:wcmdp:flow} as $yB=0$, where $B\in\R^{|\sspa||\aspa|\times|\sspa|}$ has entries
\[
    B_{(s,a),s'}=\indibrac{s=s'}-P(s,a,s').
\]
Since $B\vone^\top=0$, its rank is at most $|\sspa|-1$. Conversely, if a column vector $v$ satisfies $Bv=0$, then
\[
    v(s)=\sum_{s'\in\sspa}P(s,a,s')v(s')
    \quad\forall (s,a)\in\sspa\times\aspa.
\]
Averaging over $\pibs(a\mid s)$ gives $P_\pibs v=v$. By \Cref{assump:wcmdp:aperiodic-unichain}, the eigenspace at $1$ is spanned by $\vone^\top$. Thus $\ker B=\operatorname{span}\{\vone^\top\}$ and $\operatorname{rank}B=|\sspa|-1$.

The equation $\sumsa y(s,a)=1$ in \eqref{eq:wcmdp:simplex} adds one independent equality: its coefficient vector $\vone^\top_{|\sspa||\aspa|}$ cannot lie in the column space of $B$, because $y^*B=0$ whereas $y^*\vone^\top_{|\sspa||\aspa|}=1$. Thus the LP has $|\sspa|$ independent equality constraints.

By the definition of non-degenerate BFS, the number of tight inequality constraints at $y^*$ equals $|\sspa||\aspa|$ minus the rank of the equality constraints, $|\sspa|$, so there are $|\sspa||\aspa| - |\sspa|$ tight inequality constraints. 
With $|\sempty|=0$, the binding inequalities at $y^*$ are the $|\Kset|$ tight budget constraints and the $|\Uset|$ non-negativity constraints on $\Uset$, giving $|\Kset| + |\Uset| = |\sspa||\aspa| - |\sspa|$. Consequently, we have $|\Uset| = |\sspa|(|\aspa|-1) - |\Kset|$. Substituting into the definition of $d$, we get
\[
\begin{aligned}
d &= |\Uset| + |\Kset| + |\sspa| \\
    &= \bigl(|\sspa|(|\aspa|-1) - |\Kset|\bigr) + |\Kset| + |\sspa| \\
    &= |\sspa||\aspa|.
\end{aligned}\qedhere
\]
\end{proof}

\begin{remark}[Transient states and degeneracy]\label{rem:wcmdp:transient-degeneracy}
Under \Cref{assump:wcmdp:aperiodic-unichain}, a non-degenerate BFS cannot have transient states. To prove this claim, we consider a BFS $y^*$ with $\sempty\ne\emptyset$ and define the column vector $u$ by $u(s)=1$ for $s\in\sempty$ and $u(s)=0$ otherwise. Summing flow balance over $\sempty$ gives $yBu=0$, with $B$ as defined in the preceding proof. Since $\sempty$ is a nonempty proper subset of $\sspa$, $u$ is nonconstant, so the identity $\ker B=\operatorname{span}\{\vone^\top\}$ from that proof implies $Bu\ne0$.

The equation $yBu=0$ involves only variables $y(s,a)$ such that $y^*(s,a)=0$. Indeed, zero stationary mass in $\sempty$ and the flow-balance equations \eqref{eq:wcmdp:flow} evaluated at $y^*$ give
\[
\begin{aligned}
0 &= \sum_{s'\in\sempty}\sum_{a\in\aspa}y^*(s',a) \\
  &= \sumsa y^*(s,a)\sum_{s'\in\sempty}P(s,a,s').
\end{aligned}
\]
Nonnegativity of the terms in the last sum implies $\sum_{s'\in\sempty}P(s,a,s')=0$ whenever $y^*(s,a)>0$. For such a pair $(s,a)$, we also have $u(s)=0$, so $(Bu)_{(s,a)}=0$.

To conclude degeneracy, we replace the LP equalities by an independent spanning set of $|\sspa|$ equations, using the equality rank established in the preceding proof. The retained equalities have the same linear consequences as the original LP equalities, including $yBu=0$. The equation $yBu=0$ is also a linear combination of the tight non-negativity constraints, written as $y(s,a)=0$ at pairs $(s,a)$ where $y^*(s,a)=0$. The two expressions for the nonzero vector $Bu$, one using the retained equalities and the other using the tight non-negativity constraints, give a nontrivial linear dependence among the constraint coefficient vectors. Since the retained LP equalities are independent and $y^*$ is a BFS, this linear dependence proves that $y^*$ is degenerate.
\end{remark}

The dimensional condition $d \leq |\sspa||\aspa|$ is a necessary dimension check for \Cref{assump:wcmdp:non-degeneracy}. Full column rank requires the columns of $\Cmat$ (i.e., the constraints in \eqref{eq:wcmdp:active-system}) to be linearly independent in $\R^{|\sspa||\aspa|}$; the dimension calculation alone does not establish this independence.

\subsection{Properties of matrix \texorpdfstring{$\Phi$}{Phi} and validity of \hcref{assump:wcmdp:local-stability}}\label{subsec:wcmdp:local-stability-validity}

In this subsection, we investigate properties of the matrix $\Phi$ given by
\begin{equation}\tag{\ref{eq:wcmdp:phi-def}}
    \Phi \triangleq  \Cinvs \mathcal{P} - \vone^\top\statdist\,\Cinvs\mathcal{P}.
\end{equation}
We show that the first term, $\Cinvs \mathcal{P}$, has an eigenvalue $1$, while the second term $\vone^\top\statdist\,\Cinvs\mathcal{P}$ replaces this eigenvalue with $0$, without changing the right-multiplication on the centered probability simplex $\{v-\statdist \colon v\in\simplex(\sspa)\}$.
Therefore, the local stability assumption $\rho(\Phi)<1$ is fully determined by the behavior of the matrix $\Cinvs \mathcal{P}$ restricted to the centered probability simplex. The lemma and proposition that formalize these facts are as follows.

\begin{lemma}\label{lem:wcmdp:Cinvs-ones}
    The matrix $\Cinvs$ has row sums equal to $1$, and the square matrix $\Cinvs \mathcal{P}$ has eigenvalue $1$ with right eigenvector $\vone^\top$:
    \begin{align}
        \label{eq:Cinvs-eig}
        \Cinvs\,\vone^\top_{|\sspa||\aspa|} &= \vone^\top \\
        \label{eq:Cinvs-P-eig}
        \Cinvs\,\mathcal{P}\,\vone^\top &= \vone^\top.
    \end{align}
\end{lemma}
\begin{proof}
    We first show \eqref{eq:Cinvs-eig}.
    We recall that
    \begin{equation}\tag{\ref{eq:wcmdp:Cinvs-Cmat}}
        \Cinvs\,\Cmat = \bigl(0,\;0,\;0,\;I_{|\sspa|}\bigr).
    \end{equation}
    We let $M\in\R^{(|\sspa||\aspa|)\times|\sspa|}$ be the last $\abs{\sspa}$ columns of $\Cmat$. Then
    \[
        \Cinvs M = I_{|\sspa|}.
    \]
    Since $\Cmat$ is the coefficient matrix of the linear system \eqref{eq:wcmdp:active-system}, one can read off from the linear system that $M_{(s,a),s'} = \indibrac{s=s'}$ for $s,s'\in\sspa$ and $a\in\aspa$.
    Since $M\,\vone^\top = \vone^\top_{|\sspa||\aspa|}$, we have $\Cinvs\,\vone^\top_{|\sspa||\aspa|} = \Cinvs\,M\,\vone^\top = \vone^\top$.

    Next, we show \eqref{eq:Cinvs-P-eig}. Since $\mathcal{P}_{(s,a),s'} = P(s,a,s')$, we have $\sum_{s'}\mathcal{P}_{(s,a),s'} = 1$ for each row $(s,a)\in\sspa\times\aspa$. Consequently, we have $\mathcal{P}\,\vone^\top = \vone^\top_{|\sspa||\aspa|}$ and thus $\Cinvs\,\mathcal{P}\,\vone^\top = \Cinvs\,\vone^\top_{|\sspa||\aspa|} = \vone^\top$.
\end{proof}

\begin{proposition}\label{prop:wcmdp:phi-vs-CinvsP}
    Let $\Phi$ be defined as in \eqref{eq:wcmdp:phi-def}. Then:
    \begin{enumerate}[leftmargin=2em, label=(\roman*)]
        \item $\Phi$ and $\Cinvs\mathcal{P}$ agree on the hyperplane $H \triangleq \{v\in\R^{|\sspa|} : v\vone^\top = 0\}$, i.e., $v\Phi = v\,\Cinvs\mathcal{P}$ for all $v\in H$.
        \item The eigenvalues of $\Phi$ are exactly the eigenvalues of $\Cinvs\mathcal{P}$ with one occurrence of the eigenvalue $1$ (associated with the right eigenvector $\vone^\top$) replaced by $0$, leaving all other eigenvalues unchanged.
    \end{enumerate}
\end{proposition}

\begin{proof}
For (i): For any $v\in H$, $v(\Phi - \Cinvs\mathcal{P}) = -v\,\vone^\top\,\statdist\,\Cinvs\mathcal{P} = -(v\,\vone^\top)\,\statdist\,\Cinvs\mathcal{P} = 0$.

For (ii), we apply Brauer\textquotesingle s rank-one perturbation theorem \cite{Bra_52}. Since $\Cinvs \mathcal{P}$ has an eigenvalue $1$ with right eigenvector $\vone^\top$,
Brauer's theorem implies that
the spectrum of $\Phi =  \Cinvs \mathcal{P} - \vone^\top\statdist\,\Cinvs\mathcal{P}$ is identical to the spectrum of $\Cinvs \mathcal{P}$, with one occurrence of $1$ replaced by $1 - \statdist\,\Cinvs\mathcal{P}\vone^\top = 0$.
\end{proof}

\section{General lemmas on subroutine dynamics and weighted norms}\label{app:wcmdp:preliminary-proofs}

\yigecomment{This appendix is new.}

This appendix collects general-purpose lemmas used throughout the proofs in the WCMDP setting.
\Cref{app:wcmdp:weighted-l2-norm-lemmas} establishes well-definedness of the weight matrices $\wmat$ and $\umat$ and records the pseudo-contraction properties of $P_\pibs$ and $\Phi$ under the corresponding weighted $L_2$ norms.
\Cref{app:wcmdp:proof-of-linear-feasibility} establishes feasibility of the target state-action distribution $y = y^* + (x-\statdist)\,\Cinvs$ used by \Cref{alg:wcmdp:olc}.
Building on these, \Cref{app:wcmdp:subroutine-transition-lemmas} states the transition dynamics under the subroutines \frcontrol (\Cref{alg:wcmdp:uoc}) and \lpfp (\Cref{alg:wcmdp:olc}), and establishes an $O(1/N)$ bound on the rounding error incurred by \lpfp.

\subsection{Weighted \texorpdfstring{$L_2$}{L2} norms for quantifying the convergence of distributions}\label{app:wcmdp:weighted-l2-norm-lemmas}

The matrices $\wmat$, $\umat$ and their weighted norms are defined in \Cref{def:wcmdp:w-and-u}. We establish their well-definedness and contraction properties below.

\begin{lemma}\label{lem:wcmdp:W-U-well-defined}
    Assume \Cref{assump:wcmdp:aperiodic-unichain,assump:wcmdp:non-degeneracy,assump:wcmdp:local-stability}. The matrices $\wmat$ and $\umat$ in \Cref{def:wcmdp:w-and-u} are well-defined and positive definite, and their eigenvalues are lower bounded by $1$.
\end{lemma}

\begin{proof}
    The well-definedness and eigenvalue lower bound for $\wmat$ has been proved in \cite{HonXieCheWan_24}. The main fact used in the proof is that all eigenvalues of $P_\pibs - \vone^\top \statdist$ have moduli strictly less than $1$ under \Cref{assump:wcmdp:aperiodic-unichain}.
    Since we have assumed the same things for $\Phi$ in \Cref{assump:wcmdp:local-stability}, the proof for $\umat$ is analogous with $P_\pibs - \vone^\top \statdist$ replaced by  $\Phi$. 
\end{proof}

\begin{restatable}[Pseudo-contraction under the weighted $L_2$ norms]{lemma}{wcmdpwucontraction}\label{lem:wcmdp:one-step-contraction-W-U}
    Assume \Cref{assump:wcmdp:aperiodic-unichain,assump:wcmdp:non-degeneracy,assump:wcmdp:local-stability}. For any distribution $v\in\simplex(\sspa)$,
    \begin{align}
        \label{eq:wcmdp:pibar-contraction}
        \norm{(v - \statdist)P_\pibs}_\wmat &\leq \rhow \norm{v - \statdist}_\wmat, \\
        \label{eq:wcmdp:lp-contraction}
        \norm{(v - \statdist)\Phi}_\umat &\leq \rhou \norm{v - \statdist}_\umat,
    \end{align}
    where $\rhow = 1 - 1/(2\lamw)$ and $\rhou = 1 - 1/(2\lamu)$.
\end{restatable}

\begin{proof}
The inequality \eqref{eq:wcmdp:pibar-contraction} follows from the weighted-norm argument in \cite{HonXieCheWan_24}; the same argument applies to \eqref{eq:wcmdp:lp-contraction} with $P_\pibs-\vone^\top\statdist$ replaced by $\Phi$.
\end{proof}

\subsection{Feasibility of \lpfp}\label{app:wcmdp:proof-of-linear-feasibility}
The next lemma shows that the feasibility radius condition $\norm{x - \statdist}_\umat \leq \eta$ implies the feasibility of \lpfp.

\begin{lemma}[Inequality feasibility of the affine target]\label{lem:wcmdp:linear-feasibility}
    With $\epsilon$ as in \eqref{eq:wcmdp:eps-eta} and $\eta$ as in \eqref{eq:wcmdp:eta-def}, for every $x\in\simplex(\sspa)$ with $\norm{x - \statdist}_\umat \leq \eta$, the vector $y_t = y^* + (x-\statdist)\,\Cinvs$ defined in \eqref{eq:wcmdp:y-linear} satisfies
    \begin{enumerate}[leftmargin=2em, label=(\roman*)]
        \item \emph{Non-negativity:} $y_t(s,a) \geq 0$ for all $(s,a)\in\sspa\times\aspa$;
        \item \emph{Budget constraints:} $\sumsa c_k(s,a)\,y_t(s,a) \leq \alpha_k$ for every $k\in[K]$.
    \end{enumerate}
    Therefore, \lpfp is feasible for all arms whenever $\norm{X_t([N]) - \statdist}_\umat \leq \eta$.
\end{lemma}

\begin{proof}
    We first invoke Cauchy--Schwarz to get the following bound that is used for proving both (i) and (ii): for every $(s,a)\in\sspa\times\aspa$, we have
    \begin{equation}\label{eq:wcmdp:y_t-entrywise}
\begin{aligned}
& \abs{y_t(s,a)-y^*(s,a)} \\
    & \quad = \abs{(x-\statdist)\,\Cinvs\,e_{(s,a)}} \\
    & \quad \leq \opnorm{\Cinvs}_{\umat^{-1}}\,\norm{x-\statdist}_\umat \\
    & \quad \leq \opnorm{\Cinvs}_{\umat^{-1}}\,\eta = \epsilon.
\end{aligned}
\end{equation}

    Now we prove (i) and (ii) separately.

    \emph{(i) Non-negativity.} For $(s,a)\notin\Uset\cup(\sempty\times\aspa)$, $y^*(s,a) \geq \epsilon$ by \eqref{eq:wcmdp:eps-eta}, so \eqref{eq:wcmdp:y_t-entrywise} gives $y_t(s,a)\geq 0$. For $(s,a)\in\Uset$, $y_t(s,a) = 0$ by the $\Uset$-block of \eqref{eq:wcmdp:Cinvs-Cmat}. For $s\in\sempty$, the $\sempty$-block of \eqref{eq:wcmdp:active-system} forces $y_t(s,0) = \cdots = y_t(s,|\aspa|-1) = x(s)/|\aspa| \geq 0$.

    \emph{(ii) Budget constraints.}
    For $k\in\Kset$, the linear equations \eqref{eq:wcmdp:active-system} force $\sumsa c_k(s,a)\,y_t(s,a) = \alpha_k$.
    For $k\notin\Kset$, by the definition of $\epsilon$ in \eqref{eq:wcmdp:eps-eta}, we have
    \[
        \sumsa c_k(s,a)\,y^*(s,a) \leq \alpha_k - |\sspa||\aspa|\,\epsilon
    \]
    Combining this bound with \eqref{eq:wcmdp:y_t-entrywise} and using the fact that $c_k(s,a)\leq 1$ for all $(s,a)\in\sspa\times\aspa$, we have
    \[
\begin{aligned}
& \sumsa c_k(s,a)\,y_t(s,a) \\
    & \quad \leq \sumsa c_k(s,a)\,y^*(s,a) + |\sspa||\aspa|\,\epsilon \\
    & \quad \leq \alpha_k.
\end{aligned}\qedhere
\]
\end{proof}

\subsection{Transition dynamics under subroutines}\label{app:wcmdp:subroutine-transition-lemmas}

The lemma below characterizes the dynamics of the empirical distribution of any subset of arms following \frcontrol (\Cref{alg:wcmdp:uoc}).

\begin{lemma}[\frcontrol transition dynamics]\label{lem:wcmdp:fr-transition}
    For any time step $t$ and any $D\subseteq[N]$,
    \begin{equation}
\begin{aligned}
& \E{X_{t+1}(D)\givenplain X_t,\,\textup{arms in $D$ follow \Cref{alg:wcmdp:uoc}}} \\
    & \quad = X_t(D)\,P_\pibs \quad a.s.,
\end{aligned}
\end{equation}
    where the expectation is taken entry-wise and $P_\pibs(s,s') = \sum_{a\in\aspa} \pibs(a|s)\,P(s,a,s')$ for $s,s'\in\sspa$.
\end{lemma}

\begin{proof}
For each $s'\in\sspa$, independent action sampling and the transition kernel give
\[
\begin{aligned}
& \E{X_{t+1}(D,s')\given X_t,\ \text{arms in $D$ follow \Cref{alg:wcmdp:uoc}}} \\
    & \quad = \sumsa X_t(D,s)\pibs(a\mid s)P(s,a,s') \\
    & \quad = (X_t(D)P_\pibs)(s').
\end{aligned}
\]
\end{proof}

Next, we state two lemmas for \lpfp (\Cref{alg:wcmdp:olc}). For simplicity, we state and prove them for the case where all $N$ arms follow \lpfp. 
For a nonempty subset $D\subseteq[N]$ satisfying $\norm{X_t(D)/m(D)-\statdist}_\umat\leq\eta$, these lemmas apply verbatim to the standalone system formed by the arms in $D$, with $X_t([N])$ replaced by $X_t(D)/m(D)$ and $N$ replaced by $|D|$. For $D=\emptyset$, the subroutine makes no assignments and its state and state-action counts remain zero.

We first prove the following lemma that bounds the rounding error of \lpfp. 

\begin{lemma}[Rounding error of \lpfp]\label{lem:wcmdp:olc-rounding}
    Suppose all arms follow \lpfp (\Cref{alg:wcmdp:olc}) at time $t$, with target $y_t = y^* + (X_t([N]) - \statdist)\,\Cinvs$ as in \eqref{eq:wcmdp:y-linear}, and let $Y_t$ denote the realized state-action distribution. Then
    \[
        \norm{Y_t - y_t}_1 \leq \frac{2|\sspa|(|\aspa|-1)}{N}.
    \]
\end{lemma}

\begin{proof}
    By floor rounding on Line~\ref{alg:wcmdp:olc:assign} of \Cref{alg:wcmdp:olc}, $|N\,Y_t(s,a) - N\,y_t(s,a)| \leq 1$ for each $(s,a)$ with $a\neq 0$, so
    \[
        \sum_{a\neq 0} \abs{Y_t(s,a) - y_t(s,a)} \leq \frac{|\aspa|-1}{N}.
    \]
    For each state $s$, the arms not assigned actions $a \neq 0$ take action $0$, so $Y_t(s,0) = X_t([N],s) - \sum_{a\neq 0} Y_t(s,a)$; since $y_t(s,0) = X_t([N],s) - \sum_{a\neq 0} y_t(s,a)$ as well, we have
    \[
\begin{aligned}
& \abs{Y_t(s,0)-y_t(s,0)} \\
    & \quad \leq \sum_{a\neq0}\abs{Y_t(s,a)-y_t(s,a)} \\
    & \quad \leq \frac{|\aspa|-1}{N}.
\end{aligned}
\]
    Summing the first and second display over $s\in\sspa$ gives $\norm{Y_t - y_t}_1 \leq 2|\sspa|(|\aspa|-1)/N$.
\end{proof}

Next, we restate and prove \Cref{lem:wcmdp:olc-transition} on the transition dynamics of \lpfp. 

\wcmdpolctransition*

\begin{proof}
Since $Y_t$ is deterministic given $X_t$ (floor rounding is a deterministic function of $X_t$), the expected next state counts at each $s'\in\sspa$ are
\[
    \E{X_{t+1}([N], s') \givenplain X_t} = \sum_{s\in\sspa,\,a\in\aspa} Y_t(s,a)\,P(s,a,s').
\]
Aggregating in vector form and separating out $y_t$, we obtain
\begin{align}
    & \E{X_{t+1}([N])\givenplain X_t} \nonumber \\
    & \quad = Y_t\,\mathcal{P} \nonumber \\
    & \quad = y_t\,\mathcal{P} \;+\; (Y_t-y_t)\,\mathcal{P}. \label{eq:wcmdp:olc-decomp-step-1}
\end{align}
It remains to calculate $y_t\,\mathcal{P}$. Substituting the definition of $y_t$ from \eqref{eq:wcmdp:y-linear}, we obtain
\[
    y_t\,\mathcal{P} = y^*\,\mathcal{P} \;+\; (X_t([N]) - \statdist)\,\Cinvs\,\mathcal{P}.
\]
By the flow constraint of \eqref{eq:wcmdp-lp}, $\sum_{s,a}y^*(s,a)P(s,a,s') = \statdist(s')$ for all $s'\in\sspa$, so $y^*\mathcal{P} = \statdist$. We recall that $\Phi = \Cinvs\mathcal{P} - \vone^\top\statdist\,\Cinvs\mathcal{P}$ from \eqref{eq:wcmdp:phi-def}, so
\[
\begin{aligned}
& (X_t([N])-\statdist)\,\Cinvs\,\mathcal{P} \\
    & \quad = (X_t([N])-\statdist)\,\Phi \\
    & \qquad + (X_t([N])-\statdist)\,\vone^\top\statdist\,\Cinvs\mathcal{P} \\
    & \quad = (X_t([N])-\statdist)\,\Phi,
\end{aligned}
\]
where the last equality uses $(X_t([N]) - \statdist)\,\vone^\top = 0$. Hence $y_t\,\mathcal{P} = \statdist + (X_t([N]) - \statdist)\,\Phi$, which combined with \eqref{eq:wcmdp:olc-decomp-step-1} gives \eqref{eq:wcmdp:olc-transition}. The norm bound $\|(Y_t-y_t)\mathcal{P}\|_1 \leq \|Y_t-y_t\|_1 \leq 2|\sspa|(|\aspa|-1)/N$ follows from \Cref{lem:wcmdp:olc-rounding} and row-stochasticity of $\mathcal{P}$.
\end{proof}

\section{Detailed upper-bound proofs}\label{app:wcmdp:upper}

\subsection{Subroutine conformity}\label{sec:wcmdp:subroutine-conform}

The following lemma establishes that the two-set policy is well-defined and feasible.

\begin{restatable}[Subroutine conformity]{lemma}{wcmdpsubroutineconformity}\label{lem:wcmdp:subroutine-conform}
    Assume \Cref{assump:wcmdp:aperiodic-unichain,assump:wcmdp:non-degeneracy,assump:wcmdp:local-stability} hold. 
    Consider the two-set policy described in \Cref{alg:wcmdp:two-set}. Its set-selection steps always admit an admissible choice.
    For any time $t$, we have the following:
    \begin{enumerate}[leftmargin=2em, label=(\roman*)]
        \item The arms in $\Db_t$ can follow \lpfp.
        \item The budget constraints are satisfied, i.e., for every $k\in[K]$, $\sum_{i\in [N]} c_k(S_t(i), A_t(i)) \leq \alpha_k N$.
    \end{enumerate}
\end{restatable}

\begin{proof}[Proof of \Cref{lem:wcmdp:subroutine-conform}]
    We first consider the selection of $\Db_t$. If the full set is feasible, it is $\errtol$-maximal. Otherwise, when $\Db_{t-1}$ is feasible, the family of feasible supersets of $\Db_{t-1}$ is nonempty; when it is infeasible, the family of all feasible sets contains $\emptyset$. In either case, a set of maximum cardinality within that family is $\errtol$-maximal, since $\errtol\geq0$ and every superset with slack at least $\errtol$ is also feasible. Thus an admissible $\Db_t$ exists.

    For the selection of $\Da_t$, we let $B_t=\Da_{t-1}\setminus\Db_t$ and $q_t=\lfloor\alpha_{\min}(N-|\Db_t|)\rfloor$. Both $B_t\subseteq[N]\setminus\Db_t$ and $0\leq q_t\leq N-|\Db_t|$ hold. If $|B_t|\geq q_t$, we take a subset of $B_t$ of size $q_t$; otherwise, we extend $B_t$ to that size within $[N]\setminus\Db_t$. These choices satisfy the required nesting condition. Each set-selection step therefore samples from a finite, nonempty family of admissible choices.

    For (i), if $\Db_t = \emptyset$, \lpfp returns without assigning actions. If $\Db_t \neq \emptyset$, because the two-set policy chooses $\Db_t$ to satisfy $\slkb(X_t, \Db_t) \geq 0$, we have
    \[
        \norm{X_t(\Db_t)/m(\Db_t) - \statdist}_\umat \leq \eta.
    \]
    Since $X_t(\Db_t)/m(\Db_t)$ is the empirical distribution of arms in $\Db_t$, by \Cref{lem:wcmdp:linear-feasibility}, the arms in $\Db_t$ can follow \lpfp.

    For (ii), the arms in $\Db_t$ satisfy the scaled budget constraints
    \[
        \sum_{i\in \Db_t} c_k(S_t(i), A_t(i))
        \leq N m(\Db_t)\,\alpha_k.
    \]
    This follows from \Cref{lem:wcmdp:linear-feasibility} when $\Db_t\ne\emptyset$; for an empty set, both sides are zero.
    The cost in $\Da_t$ is at most $|\Da_t| \leq N \alpha_{\min}(1 - m(\Db_t)) \leq N \alpha_k(1 - m(\Db_t))$ since $c_k(s,a) \in [0,1]$ for all $s\in\sspa$, $a\in\aspa$, and $k\in[K]$. 
    Since the remaining arms take action $0$, by the assumption that $c_k(s,0) = 0$, they do not contribute to the total cost. Therefore, summing the cost of the three sets of arms, we get
    \[
\begin{aligned}
& \sum_{i\in[N]}c_k(S_t(i),A_t(i)) \\
    & \quad \leq N m(\Db_t)\,\alpha_k + N\alpha_k(1-m(\Db_t)) + 0 \\
    & \quad = N\alpha_k.
\end{aligned}\qedhere
\]
\end{proof}

\subsection{Proof of \hcref{lem:wcmdp:Vds-properties}}\label{sec:wcmdp:pf-Vds-properties}

\wcmdpvdsproperties*

The proof uses three supporting lemmas: subset concentration, almost non-shrinking, and sufficient coverage. We state them first, then give the composite drift argument and supporting proofs.

\begin{restatable}[Properties of $\hw$ and $\hu$]{lemma}{wcmdphwhuproperties}\label{lem:wcmdp:hw-hu-properties}
    Assume \Cref{assump:wcmdp:aperiodic-unichain,assump:wcmdp:non-degeneracy,assump:wcmdp:local-stability} hold. For any $t$ and $D\subseteq[N]$, let $X_{t+1}'$ be the system state at time $t+1$ if the arms in $D$ follow \Cref{alg:wcmdp:uoc}. Then
    \begin{align}
    & \Ebig{(\hw(X_{t+1}',D)-\rhow\,\hw(X_t,D))^+\givenbig X_t} \nonumber \\
    & \quad \leq \frac{\KhwOne}{\sqrt{N}} \quad a.s., \label{eq:wcmdp:hw:drift} \\
    & \Probbig{\hw(X_{t+1}',D)>\rhow\,\hw(X_t,D)+r\givenbig X_t} \nonumber \\
    & \quad \leq \KhwTwo\exp(-\Chw N r^2) \quad a.s.,\;\forall r\geq0, \label{eq:wcmdp:hw:high-prob}
\end{align}
    where $\rhow = 1 - 1/(2\lamw)$ and $\KhwOne, \KhwTwo, \Chw$ are positive constants independent of $N$. Moreover, for any $D, D' \subseteq [N]$,
    \begin{align}
    & \hw(x,D)\geq\frac{1}{|\sspa|^{1/2}}\norm{x(D)-m(D)\statdist}_1, \label{eq:wcmdp:hw:strength} \\
    & \abs{\hw(x,D)-\hw(x,D')} \nonumber \\
    & \quad \leq 2\lamw^{1/2}\bigl(m(D'\setminus D)+m(D\setminus D')\bigr). \label{eq:wcmdp:hw:lipschitz}
\end{align}
    For $\hu$, take a nonempty subset $D\subseteq[N]$ and suppose that
    \[
        \norm{X_t(D)/m(D)-\statdist}_\umat\leq\eta.
    \]
    If the arms in $D$ follow \Cref{alg:wcmdp:olc}, then $\hu$ satisfies the analogous drift bound \eqref{eq:wcmdp:hw:drift} and tail bound \eqref{eq:wcmdp:hw:high-prob}, with constants $\KhuOne,\KhuTwo,\Chu$ and contraction rate $\rhou=1-1/(2\lamu)$. Here $X_{t+1}'$ denotes the resulting next state. For $D=\emptyset$, both subset functions vanish and these drift and tail bounds hold trivially. The strength and Lipschitz inequalities \eqref{eq:wcmdp:hw:strength}--\eqref{eq:wcmdp:hw:lipschitz} hold for $\hu$ for all $D,D'\subseteq[N]$, with $\lamw$ replaced by $\lamu$.
\end{restatable}

\begin{restatable}[Almost non-shrinking]{lemma}{wcmdpalmostnonshrinking}\label{lem:wcmdp:non-shrink}
    Under \Cref{assump:wcmdp:aperiodic-unichain,assump:wcmdp:non-degeneracy,assump:wcmdp:local-stability}, we have that for any $t\geq 0$, 
    \begin{align*}
    & \Ebig{m(\Db_t\setminus\Db_{t+1})\givenbig\Fullstate_t} \\
    & \quad \leq \KnsOne\,m(\Db_t)\exp(-\Cns N m(\Db_t)^2) \quad a.s., \\
    & \Probbig{m(\Db_t\setminus\Db_{t+1})>0\givenbig\Fullstate_t} \\
    & \quad \leq \KnsTwo\exp(-\Cns N m(\Db_t)^2) \quad a.s.,
\end{align*}
    where $\KnsOne, \KnsTwo, \Cns > 0$ are constants independent of $N$.
\end{restatable}

\begin{restatable}[Sufficient coverage]{lemma}{wcmdpsufficientcoverage}\label{lem:wcmdp:sufficient-coverage}
    Assuming \Cref{assump:wcmdp:aperiodic-unichain,assump:wcmdp:non-degeneracy,assump:wcmdp:local-stability} and $\errtol=O(1/N)$, we have that for any $t\geq 0$, 
    \begin{align*}
        1 - m(\Db_t) &\leq \KscOne\,\hw(X_t, \Da_t) + \frac{\KscTwo}{N} \quad a.s., \\
        m(\Db_t) &= 1 \quad \text{ if } \hu(X_t, [N]) \leq \eta,
    \end{align*}
    where $\KscOne, \KscTwo > 0$ are constants independent of $N$.
\end{restatable}

To prove \Cref{lem:wcmdp:Vds-properties}, we apply the drift and tail bounds \eqref{eq:wcmdp:hw:drift}--\eqref{eq:wcmdp:hw:high-prob} and their $\hu$ analogues conditional on $\Fullstate_t$, with $D=\Da_t$ for $\hw$ and $D=\Db_t$ for $\hu$.

\begin{proof}[Proof of \Cref{lem:wcmdp:Vds-properties}]
    The domination and preliminary drift calculation in \Cref{app:wcmdp:lyapunov-calculations} gives the following bound, with $\rhoFinal\in(0,1)$ and $\Kdrift>0$ as defined there:
    \begin{align}
    & \Vds(\Fullstate_{t+1})-\rhoFinal\Vds(\Fullstate_t) \nonumber \\
    & \quad \leq \big(\hu(X_{t+1},\Db_t)-\rhou\hu(X_t,\Db_t)\big) \nonumber \\
    & \qquad + \big(\hw(X_{t+1},\Da_t)-\rhow\hw(X_t,\Da_t)\big) \label{eq:two-set:thm:v-diff-final-term-1-wcmdp} \\
    & \qquad + 4\big(\lamu^{1/2}+\lamw^{1/2}\big)m(\Db_t\backslash\Db_{t+1}) \nonumber \\
    & \qquad + \frac{\Kdrift}{N}. \label{eq:two-set:thm:v-diff-final-term-2-wcmdp}
\end{align}
    It remains to prove \eqref{eq:wcmdp:Vds-mean} and \eqref{eq:wcmdp:Vds-tail} using \eqref{eq:two-set:thm:v-diff-final-term-1-wcmdp}--\eqref{eq:two-set:thm:v-diff-final-term-2-wcmdp}. 

    \medskip
    \textbf{Proof of \eqref{eq:wcmdp:Vds-mean}.}
    To prove \eqref{eq:wcmdp:Vds-mean}, we take the conditional expectation of $\Bigl(\Vds(\Fullstate_{t+1}) - \rhoFinal\,\Vds(\Fullstate_t) \Bigr)^{\!+}$ and apply \eqref{eq:two-set:thm:v-diff-final-term-1-wcmdp}--\eqref{eq:two-set:thm:v-diff-final-term-2-wcmdp}, which yields
    \begin{align*}
    & \E{\Bigl(\Vds(\Fullstate_{t+1})-\rhoFinal\,\Vds(\Fullstate_t)\Bigr)^{\!+}\given\Fullstate_t=\fullstate} \\
    & \quad \leq \mathbb{E}\Bigl[\big(\hu(X_{t+1},\Db_t)-\rhou\hu(X_t,\Db_t)\big)^+ \\
    & \qquad + \big(\hw(X_{t+1},\Da_t)-\rhow\hw(X_t,\Da_t)\big)^+ \\
    & \qquad + 4\big(\lamu^{1/2}+\lamw^{1/2}\big)m(\Db_t\backslash\Db_{t+1}) \givenBig\Fullstate_t=\fullstate\Bigr] + \frac{\Kdrift}{N} \\
    & \quad \leq \frac{\KhuOne+\KhwOne}{\sqrt{N}} +\frac{4\big(\lamu^{1/2}+\lamw^{1/2}\big)\KnsOne}{\sqrt{2e\Cns N}} + \frac{\Kdrift}{N},
\end{align*}
    where the last step uses \Cref{lem:wcmdp:hw-hu-properties,lem:wcmdp:non-shrink} and
    \[
        \sup_{m\geq0}m\exp(-\Cns Nm^2)=\frac{1}{\sqrt{2e\Cns N}}.
    \]
    Since $1/N\leq1/\sqrt{N}$, we obtain \eqref{eq:wcmdp:Vds-mean} by choosing
    \[
    \begin{aligned}
        \KVds &= \KhuOne+\KhwOne \\
        &\quad + \frac{4\big(\lamu^{1/2}+\lamw^{1/2}\big)\KnsOne}{\sqrt{2e\Cns}}+\Kdrift.
    \end{aligned}
    \]

    \medskip
    \textbf{Proof of \eqref{eq:wcmdp:Vds-tail}.}
    Let $r_0\triangleq(1-\rhoFinal)\threshbar/2$ and $M\triangleq2\lamu^{1/2}+2\lamw^{1/2}+\Lds$. The bounds $\norm{x(D)-m(D)\statdist}_1\leq2m(D)\leq2$ imply $\Vds(\fullstate)\leq M$ for every $N$ and $\fullstate\in\fullssp$. Since $\bigl(\Vds(\Fullstate_{t+1}) - \rhoFinal\Vds(\Fullstate_t) - r_0\bigr)^{\!+} \leq M$ a.s., we have
    \begin{equation}\label{eq:wcmdp:Vds-tail:pf-prob-to-exp}
\begin{aligned}
& \E{\Bigl(\Vds(\Fullstate_{t+1})-\rhoFinal\Vds(\Fullstate_t)-r_0\Bigr)^{\!+}\given\Fullstate_t} \\
    & \quad \leq M\cdot\ProbBig{\Vds(\Fullstate_{t+1})>\rhoFinal\Vds(\Fullstate_t)+r_0\givenBig\Fullstate_t} \quad a.s.,
\end{aligned}
\end{equation}
    so it suffices to show the probability on the right decays as $\exp(-CN)$ for some $C > 0$ independent of $N$ and $\fullstate$.

    By \eqref{eq:two-set:thm:v-diff-final-term-1-wcmdp}--\eqref{eq:two-set:thm:v-diff-final-term-2-wcmdp} and the union bound,
    \begin{align}
    & \ProbBig{\Vds(\Fullstate_{t+1})>\rhoFinal\Vds(\Fullstate_t)+r_0\givenBig\Fullstate_t} \nonumber \\
    & \quad \leq \Probbig{\hu(X_{t+1},\Db_t)>\rhou\hu(X_t,\Db_t)+r_0/4\givenbig\Fullstate_t} \nonumber \\
    & \qquad + \Probbig{\hw(X_{t+1},\Da_t)>\rhow\hw(X_t,\Da_t)+r_0/4\givenbig\Fullstate_t} \nonumber \\
    & \qquad + \Probbig{4\big(\lamu^{1/2}+\lamw^{1/2}\big)m(\Db_t\backslash\Db_{t+1})>r_0/4\givenbig\Fullstate_t} \nonumber \\
    & \qquad + \indibrac{\Kdrift/N>r_0/4}. \label{eq:wcmdp:Vds-tail:pf-union-bound}
\end{align}
    The first two probabilities are bounded by \Cref{lem:wcmdp:hw-hu-properties} at the fixed threshold $r_0/4$. For the third probability, the event is impossible if $m(\Db_t)\leq r_0/[16(\lamu^{1/2}+\lamw^{1/2})]$. Otherwise, the event implies that at least one old arm is lost, so \Cref{lem:wcmdp:non-shrink} bounds its probability by
    \[
        \KnsTwo\exp\!\left(-\Cns N\left[\frac{r_0}{16(\lamu^{1/2}+\lamw^{1/2})}\right]^2\right).
    \]
    The last indicator vanishes for $N>4\Kdrift/r_0$. Combining these bounds gives constants $\tilde K,\tilde C>0$, independent of $N$ and $\fullstate$, such that for $N>4\Kdrift/r_0$,
    \begin{equation}\label{eq:wcmdp:Vds-tail:pf-prob-bound}
\begin{aligned}
& \ProbBig{\Vds(\Fullstate_{t+1})>\rhoFinal\Vds(\Fullstate_t)+r_0\givenBig\Fullstate_t} \\
    & \quad \leq \tilde K\exp\!\big(-\tilde C N\big) \quad a.s.
\end{aligned}
\end{equation}
    For $1\leq N\leq4\Kdrift/r_0$, the probability is at most $1$. Increasing $\tilde K$ to be at least $\exp(4\tilde C \Kdrift/r_0)$ therefore makes \eqref{eq:wcmdp:Vds-tail:pf-prob-bound} valid for every $N\geq1$. Substituting this bound into \eqref{eq:wcmdp:Vds-tail:pf-prob-to-exp} and setting $\KVdst=M\tilde K$ and $\CVds=\tilde C$ establishes \eqref{eq:wcmdp:Vds-tail}.
\end{proof}

\begin{proof}[Proof of \Cref{lem:wcmdp:hw-hu-properties}]
    The drift bound \eqref{eq:wcmdp:hw:drift} and tail bound \eqref{eq:wcmdp:hw:high-prob}, together with their $\hu$ analogues, are immediate for $D=\emptyset$. For the rest of the proof, we assume $D\ne\emptyset$; for $\hu$, we also assume the feasibility condition stated in the lemma.

    \textbf{Proving \eqref{eq:wcmdp:hw:drift} and \eqref{eq:wcmdp:hw:high-prob}.}
    The subroutine \frcontrol (\Cref{alg:wcmdp:uoc}) samples each arm\textquotesingle s action independently, so there is no rounding error to control.

    By the pseudo-contraction property of $P_\pibs$ under the $\wmat$-weighted norm given in \Cref{lem:wcmdp:one-step-contraction-W-U},
    \begin{align*}
    & \rhow\,\hw(X_t,D) \\
    & \quad = \rhow\norm{X_t(D)-m(D)\statdist}_\wmat \\
    & \quad \geq \norm{(X_t(D)-m(D)\statdist)\,P_\pibs}_\wmat \\
    & \quad = \norm{X_t(D)\,P_\pibs-m(D)\statdist}_\wmat,
\end{align*}
    so by the triangle inequality and the relation between the $\wmat$-weighted norm and the $L_1$ norm,
    \[
\begin{aligned}
& \hw(X_{t+1}',D)-\rhow\,\hw(X_t,D) \\
    & \quad \leq \norm{X_{t+1}'(D)-X_t(D)\,P_\pibs}_\wmat \\
    & \quad \leq \lamw^{1/2}\norm{X_{t+1}'(D)-X_t(D)\,P_\pibs}_1.
\end{aligned}
\]
    Therefore, it suffices to upper bound
\[\norm{X_{t+1}'(D) - X_t(D)\,P_\pibs}_1\] in expectation and in probability.

    Under \Cref{alg:wcmdp:uoc}, each arm $i\in D$ independently samples an action $A_t(i) \sim \pibs(\cdot \,|\, S_t(i))$ and then transitions according to $P$. Hence, given $X_t$, the indicators $\indibrac{S_{t+1}'(i) = s}$ for $i\in D$ are independent Bernoulli random variables with means $P_\pibs(S_t(i), s)$, where $S_{t+1}'(i)$ denotes the state of arm $i$ at time $t+1$ if the arms in $D$ follow \Cref{alg:wcmdp:uoc}. In particular, by \Cref{lem:wcmdp:fr-transition}, $\E{X_{t+1}'(D) \given X_t} = X_t(D)\,P_\pibs$. 

    For each $s\in\sspa$, \[X_{t+1}'(D, s) = (1/N)\sum_{i\in D} \indibrac{S_{t+1}'(i) = s}\] is a sum of independent random variables taking values in $[0, 1/N]$. By Cauchy--Schwarz,
    \[
\begin{aligned}
& \EBig{\absBig{X_{t+1}'(D,s)-\E{X_{t+1}'(D,s)\given X_t}}\givenBig X_t} \\
    & \quad \leq \Var{X_{t+1}'(D,s)\givenBig X_t}^{1/2} \\
    & \quad \leq \frac{|D|^{1/2}}{N}\leq\frac{1}{\sqrt{N}},
\end{aligned}
\]
    and by Hoeffding's inequality, for all $r\geq 0$,
    \[
\begin{aligned}
& \ProbBig{\absBig{X_{t+1}'(D,s)-\E{X_{t+1}'(D,s)\givenplain X_t}}>r\givenBig X_t} \\
    & \quad \leq 2\exp\big(-2N r^2\big).
\end{aligned}
\]
    Summing the expectation bound over $s\in\sspa$, and applying the union bound over $s\in\sspa$ (with $r$ replaced by $r/|\sspa|$) to the probability bound, we get
    \begin{align*}
    & \Ebig{\norm{X_{t+1}'(D)-X_t(D)\,P_\pibs}_1\givenbig X_t} \\
    & \quad \leq \frac{|\sspa|}{\sqrt{N}} \quad a.s., \\
    & \Probbig{\norm{X_{t+1}'(D)-X_t(D)\,P_\pibs}_1>r\givenbig X_t} \\
    & \quad \leq 2|\sspa|\exp\Big(-\frac{2N r^2}{|\sspa|^2}\Big) \quad a.s.,\;\forall r\geq0.
\end{align*}
    Combining the two displays with the bound $\hw(X_{t+1}', D) - \rhow\,\hw(X_t, D) \leq \lamw^{1/2}\norm{X_{t+1}'(D) - X_t(D)\,P_\pibs}_1$ established earlier, we obtain \eqref{eq:wcmdp:hw:drift} and \eqref{eq:wcmdp:hw:high-prob} with $\KhwOne = \lamw^{1/2}|\sspa|$, $\KhwTwo = 2|\sspa|$, and $\Chw = 2/(\lamw|\sspa|^2)$.

    \textbf{Proving the analogs of \eqref{eq:wcmdp:hw:drift} and \eqref{eq:wcmdp:hw:high-prob} for $\hu$.}
    For $\hu$, we must account for the rounding residual in \Cref{lem:wcmdp:olc-transition}. We write $K_{\mathrm{rnd}}\triangleq 2|\sspa|(|\aspa|-1)$ for the constant in the rounding bound of \Cref{lem:wcmdp:olc-rounding}.

    By the pseudo-contraction property of $\Phi$ under the $\umat$-weighted norm given in \Cref{lem:wcmdp:one-step-contraction-W-U},
    \begin{align*}
    \rhou\hu(X_t,D) & = \rhou\norm{X_t(D)-m(D)\statdist}_\umat \\
    & \quad \geq \norm{(X_t(D)-m(D)\statdist)\,\Phi}_\umat.
\end{align*}
    Consequently,
    \begin{align*}
    & \hu(X_{t+1}',D)-\rhou\hu(X_t,D) \\
    & \quad \leq \norm{X_{t+1}'(D)-m(D)\statdist}_\umat \\
    & \qquad - \norm{(X_t(D)-m(D)\statdist)\,\Phi}_\umat \\
    & \quad \leq \norm{X_{t+1}'(D)-m(D)\statdist-(X_t(D)-m(D)\statdist)\,\Phi}_\umat \\
    & \quad \leq \lamu^{1/2}\Bigl\lVert X_{t+1}'(D)-m(D)\statdist \\
    & \qquad - (X_t(D)-m(D)\statdist)\,\Phi\Bigr\rVert_1.
\end{align*}
    Therefore, it suffices to show that, with $C'=1/(2|\sspa|^2)$, we have
    \begin{align}
    & \mathbb{E}\Bigl[\norm{X_{t+1}'(D)-m(D)\statdist-(X_t(D)-m(D)\statdist)\,\Phi}_1 \nonumber \\
    & \qquad \givenplain X_t\Bigr] \nonumber \\
    & \quad = O(1/\sqrt{N}) \quad a.s., \label{eq:wcmdp:hu:interm-goal-1} \\
    & \mathbb{P}\Bigl[\Bigl\lVert X_{t+1}'(D)-m(D)\statdist \nonumber \\
    & \qquad - (X_t(D)-m(D)\statdist)\,\Phi\Bigr\rVert_1>r\givenplain X_t\Bigr] \nonumber \\
    & \quad = O(\exp(-C'N r^2)) \quad a.s.,\;\forall r\geq0. \label{eq:wcmdp:hu:interm-goal-2}
\end{align}

    To control the rounding residual, we define the LP target for the arms in $D$ by
    \[
        y_t^D\triangleq y^*+\bigl(X_t(D)/m(D)-\statdist\bigr)\Cinvs.
    \]
    Applying \Cref{lem:wcmdp:olc-transition} to the arms in $D$ and multiplying by $m(D)$ gives
    \begin{equation}\label{eq:wcmdp:hu:cond-mean}
\begin{aligned}
& \E{X_{t+1}'(D)-m(D)\statdist\given X_t,(A_t(i))_{i\in D}} \\
    & \quad = (X_t(D)-m(D)\statdist)\,\Phi \\
    & \qquad + (Y_t(D)-m(D)\,y_t^D)\,\mathcal{P}.
\end{aligned}
\end{equation}
    We decompose $X_{t+1}'(D) - m(D)\statdist - (X_t(D) - m(D)\statdist)\,\Phi$ by adding and subtracting the conditional mean $\E{X_{t+1}'(D) \givenplain X_t, (A_t(i))_{i\in D}}$:
    \begin{align}
    & X_{t+1}'(D)-m(D)\statdist-(X_t(D)-m(D)\statdist)\,\Phi \nonumber \\
    & \quad = \Bigl(X_{t+1}'(D)-\E{X_{t+1}'(D)\givenplain X_t,(A_t(i))_{i\in D}}\Bigr) \nonumber \\
    & \qquad + \Bigl(\E{X_{t+1}'(D)\givenplain X_t,(A_t(i))_{i\in D}} \nonumber \\
    & \qquad - m(D)\statdist-(X_t(D)-m(D)\statdist)\,\Phi\Bigr) \nonumber \\
    & \quad = \bigl(X_{t+1}'(D)-\E{X_{t+1}'(D)\givenplain X_t,(A_t(i))_{i\in D}}\bigr) \nonumber \\
    & \qquad + (Y_t(D)-m(D)\,y_t^D)\,\mathcal{P}. \nonumber
\end{align}
    The rounding bound for $|D|$ arms, scaled by $m(D)=|D|/N$, is $K_{\mathrm{rnd}}/N$. Since $\mathcal{P}$ is row-stochastic,
    \[
    \begin{aligned}
    & \norm{(Y_t(D)-m(D)y_t^D)\mathcal{P}}_1 \\
    & \quad \leq \norm{Y_t(D)-m(D)y_t^D}_1
        \leq \frac{K_{\mathrm{rnd}}}{N}.
    \end{aligned}
    \]
    The triangle inequality therefore gives
    \begin{equation}\label{eq:wcmdp:hu:decomp}
\begin{aligned}
& \norm{X_{t+1}'(D)-m(D)\statdist-(X_t(D)-m(D)\statdist)\,\Phi}_1 \\
    & \quad \leq \norm{X_{t+1}'(D)-\E{X_{t+1}'(D)\givenplain X_t,(A_t(i))_{i\in D}}}_1 \\
    & \qquad + \frac{K_{\mathrm{rnd}}}{N}.
\end{aligned}
\end{equation}
    Given $X_t$ and $(A_t(i))_{i\in D}$, the next-state indicators are independent across arms. For each state, the centered coordinate on the right-hand side of \eqref{eq:wcmdp:hu:decomp} is a sum of $|D|$ centered Bernoulli variables divided by $N$. The same variance estimate used for $\hw$ bounds the expected $L_1$ norm by $|\sspa|\sqrt{|D|}/N\leq |\sspa|/\sqrt{N}$. Averaging over the actions and adding the rounding bound gives an upper bound of $(|\sspa|+K_{\mathrm{rnd}})/\sqrt{N}$ in \eqref{eq:wcmdp:hu:interm-goal-1}.

    To prove the tail bound, first suppose $r\geq 2K_{\mathrm{rnd}}/N$. By \eqref{eq:wcmdp:hu:decomp}, a deviation larger than $r$ requires the centered $L_1$ norm to exceed $r-K_{\mathrm{rnd}}/N\geq r/2$. Hoeffding's inequality and a union bound over states give
    \begin{align*}
    & \mathbb{P}\Bigl[\Bigl\lVert X_{t+1}'(D)-m(D)\statdist - (X_t(D)-m(D)\statdist)\,\Phi\Bigr\rVert_1>r\\
    & \qquad\qquad   \givenBig X_t,(A_t(i))_{i\in D}\Bigr] \\
    & \quad \leq \mathbb{P}\Bigl[\Bigl\lVert X_{t+1}'(D) - \E{X_{t+1}'(D)\givenplain X_t,(A_t(i))_{i\in D}}\Bigr\rVert_1 > r/2 \\
    & \qquad\qquad \givenBig X_t,(A_t(i))_{i\in D}\Bigr] \\
    & \quad \leq 2|\sspa|\exp\!\left(-\frac{N^2r^2}{2|D||\sspa|^2}\right) \\
    & \quad \leq 2|\sspa|\exp(-C'Nr^2).
\end{align*}
    For $0\leq r<2K_{\mathrm{rnd}}/N$, we instead use the probability bound $1$. Since $Nr^2\leq 4K_{\mathrm{rnd}}^2$ for $N\geq1$, this bound is at most $\exp(4C'K_{\mathrm{rnd}}^2)\exp(-C'Nr^2)$. Thus, setting
    \[
        \KhuTwo=\max\bigl\{2|\sspa|,\exp(4C'K_{\mathrm{rnd}}^2)\bigr\}
    \]
    and averaging over the actions gives, for every $r\geq0$,
    \begin{equation}\label{eq:pf-hu-drift-tail-1}
\begin{aligned}
& \mathbb{P}\Bigl[\Bigl\lVert X_{t+1}'(D)-m(D)\statdist \\
    & \qquad - (X_t(D)-m(D)\statdist)\,\Phi\Bigr\rVert_1>r\givenBig X_t\Bigr] \\
    & \quad \leq \KhuTwo\exp(-C'Nr^2) \quad a.s.
\end{aligned}
\end{equation}
    This proves \eqref{eq:wcmdp:hu:interm-goal-2}. Using the factor $\lamu^{1/2}$ in the earlier bound for $\hu(X_{t+1}',D)-\rhou\hu(X_t,D)$ establishes the claimed drift and tail bounds with
    \[
        \KhuOne=\lamu^{1/2}(|\sspa|+K_{\mathrm{rnd}}),
        \qquad \Chu=\frac{C'}{\lamu}.
    \]
    All constants are independent of $N$, $D$, and the current state.

    The inequalities \eqref{eq:wcmdp:hw:strength}--\eqref{eq:wcmdp:hw:lipschitz} and their $\hu$ analogs follow from the definitions of the weighted norms; the corresponding norm estimates are also discussed in \cite{HonXieCheWan_24}. \halmos
\end{proof}

\begin{proof}[Proof of \Cref{lem:wcmdp:non-shrink}]
    We condition throughout on $\Fullstate_t$, which fixes $X_t$ and $\Db_t$. If $\Db_t=\emptyset$, both quantities in the lemma are zero. We therefore assume $\Db_t\ne\emptyset$.
    By the definition of $\Db_{t+1}$ in \Cref{alg:wcmdp:two-set}, if $\slkb(X_{t+1},\Db_t)\geq0$, then $\Db_{t+1}$ contains $\Db_t$. Thus $\Db_t\backslash\Db_{t+1}\ne\emptyset$ can occur only when $\slkb(X_{t+1},\Db_t)<0$. Since $m(\Db_t\backslash\Db_{t+1})\leq m(\Db_t)$, we have
    \begin{align}
    & \Ebig{m(\Db_t\backslash\Db_{t+1})\givenbig\Fullstate_t} \nonumber \\
    & \quad \leq m(\Db_t)\Probbig{\slkb(X_{t+1},\Db_t)<0\givenbig\Fullstate_t} \label{eq:non-shrink:mean-interm-bound} \\
    & \Probbig{m(\Db_t\backslash\Db_{t+1})>0\givenbig\Fullstate_t} \nonumber \\
    & \quad \leq \Probbig{\slkb(X_{t+1},\Db_t)<0\givenbig\Fullstate_t}. \label{eq:non-shrink:tail-interm-bound}
\end{align}
    To bound the probability that the old set becomes infeasible, we recall from \eqref{eq:wcmdp:slack-def} that $\slkb(x,D)=\eta m(D)-\hu(x,D)$. The policy chooses $\Db_t$ to be feasible at time $t$, so $\hu(X_t,\Db_t)\leq\eta m(\Db_t)$. Thus
    \begin{align}
    & \Prob{\slkb(X_{t+1},\Db_t)<0\givenbig\Fullstate_t} \nonumber \\
    & \quad = \Prob{\hu(X_{t+1},\Db_t)>\eta m(\Db_t)\givenbig\Fullstate_t} \nonumber \\
    & \quad \leq \mathbb{P}\Bigl[\hu(X_{t+1},\Db_t)-\rhou\hu(X_t,\Db_t) \nonumber \\
    & \qquad\qquad > (1-\rhou)\eta m(\Db_t)\givenBig\Fullstate_t\Bigr] \nonumber \\
    & \quad \leq \KhuTwo\exp\!\big(-\Chu N(1-\rhou)^2\eta^2m(\Db_t)^2\big). \label{eq:pf-almost-nonshrink:interm-1}
\end{align}
    The last inequality applies the $\hu$ tail bound of \Cref{lem:wcmdp:hw-hu-properties} at $r=(1-\rhou)\eta m(\Db_t)$. Conditional on $\Fullstate_t$, the arms in the fixed set $\Db_t$ follow \lpfp with fresh randomness, so the lemma's conditional calculation applies. Combining \eqref{eq:pf-almost-nonshrink:interm-1} with \eqref{eq:non-shrink:mean-interm-bound}--\eqref{eq:non-shrink:tail-interm-bound} proves both claims with
    \[
        \KnsOne=\KnsTwo=\KhuTwo,
        \qquad \Cns=\Chu(1-\rhou)^2\eta^2.
    \]
    These constants are positive and independent of $N$ and the current state.
\end{proof}

Finally, we restate and prove \Cref{lem:wcmdp:sufficient-coverage}.

\wcmdpsufficientcoverage*

\begin{proof}[Proof of \Cref{lem:wcmdp:sufficient-coverage}]
    The second claim of \Cref{lem:wcmdp:sufficient-coverage} follows directly from the definition of the policy, so we only need to prove its first claim.
    We claim that either $m(\Da_t) \leq \errtol$, or
    \begin{equation}
        \label{eq:two-set:sufficient-coverage:interm-goal-1}
        \lamu^{1/2} \hw(X_t, \Da_t) >  \eta m(\Da_t) - \errtol - 0.
    \end{equation}
    We first show \Cref{lem:wcmdp:sufficient-coverage} assuming this claim.
    We recall that $m(\Da_{t})= \big\lfloor \beta (N -|\Db_{t}|)\big\rfloor \big/ N$. If $m(\Da_t) \leq \errtol$,
    \[
        \beta(1-m(\Db_t)) - \frac{1}{N} \leq m(\Da_t) \leq \errtol,
    \]
    so $1-m(\Db_t) \leq 1/(\beta N) + \errtol/\beta$, which implies \Cref{lem:wcmdp:sufficient-coverage} by the non-negativity of $\hw(X_t, \Da_t)$.
    If $m(\Da_t) > \errtol$, \eqref{eq:two-set:sufficient-coverage:interm-goal-1} holds. Because $m(\Da_t) = \floor{\beta (N - |\Db_t|)} /N \geq \big(\beta (N - |\Db_t|) - 1\big) / N$, we have
    \begin{align*}
    & \lamu^{1/2}\hw(X_t,\Da_t) \\
    & \quad > \eta\big(\beta(N-|\Db_t|)-1\big)\frac{1}{N}-\errtol-0 \\
    & \quad = \eta\beta(1-m(\Db_t))-\frac{\eta}{N}-\errtol-0.
\end{align*}
    Rearranging gives
    \[
        1 - m(\Db_t) < \frac{\lamu^{1/2}}{\eta\beta} \hw(X_t, \Da_t) + \frac{1}{\beta N} + \frac{\errtol+0}{\eta\beta}.
    \]
    Combining the two cases, we have
    \[
\begin{aligned}
1-m(\Db_t) &\leq \frac{\lamu^{1/2}}{\eta\beta}\hw(X_t,\Da_t) + \frac{1}{\beta N} \\
    &\quad + \frac{\errtol+0}{\min(\eta,1)\beta}.
\end{aligned}
\]
    Since $\errtol, 0 = O(1/N)$, the additive term is $O(1/N)$, and we can write the bound as $\KscOne \hw(X_t, \Da_t) + \KscTwo/N$ for positive constants $\KscOne, \KscTwo$ independent of $N$, establishing \Cref{lem:wcmdp:sufficient-coverage}.

    Now we prove the claim by contradiction. We suppose that, at a certain time $t$, we have $m(\Da_t) > \errtol$ and $ \lamu^{1/2}\hw(X_t, \Da_t) \leq \eta m(\Da_t) - \errtol - 0$. Because $\norm{v}_\umat \leq \lamu^{1/2}\norm{v}_2 \leq \lamu^{1/2}\norm{v}_\wmat$ for any $v\in \R^{|\sspa|}$,
    \begin{align}
    & \normbig{X_t(\Da_t)-m(\Da_t)\statdist}_\umat \nonumber \\
    & \quad \leq \lamu^{1/2}\normbig{X_t(\Da_t)-m(\Da_t)\statdist}_\wmat \nonumber \\
    & \quad \leq \eta m(\Da_t)-\errtol-0. \nonumber
\end{align}
    Combined with the triangular inequality and the definition of $\Db_t$, we have
    \begin{align*}
    & \normbig{X_t(\Db_t\cup\Da_t)-m(\Db_t\cup\Da_t)\statdist}_\umat \\
    & \quad \leq \normbig{X_t(\Db_t)-m(\Db_t)\statdist}_\umat \\
    & \qquad + \normbig{X_t(\Da_t)-m(\Da_t)\statdist}_\umat \\
    & \quad \leq \eta m(\Db_t)+\eta m(\Da_t)-\errtol-0.
\end{align*}
    Consequently, $\Db_t\cup\Da_t$ is a superset of $\Db_t$ such that $\slkb(X_t, \Db_t\cup\Da_t) \geq \errtol$ and $m(\Db_t \cup \Da_t) = m(\Db_t) + m(\Da_t) > m(\Db_t) + \errtol$, contradicting the $\errtol$-maximality of $\Db_t$.
    We have thus proved the claim that implies \Cref{lem:wcmdp:sufficient-coverage}.
\end{proof}

\subsection{Proof of \hcref{lem:wcmdp:inst-reward}}\label{sec:wcmdp:pf-inst-reward}

\wcmdpinstrew*

\begin{proof}[Proof of \Cref{lem:wcmdp:inst-reward}]
    When $\norm{x([N]) - \statdist}_\umat \leq \threshbar$, \lpfp is feasible for all arms. Since the two-set policy chooses $\Db_t$ to be a maximal set that can follow \lpfp, we have $\Db_t = [N]$ and $\Da_t = \emptyset$. 
    Letting $y_t = y^* + (x([N]) - \statdist)\Cinvs$, we perform the following decomposition of the instantaneous expected reward $r^\pi(\fullstate)$:
    \begin{align*}
    r^\pi(\fullstate) 
    &= \E{\sumsa r(s,a)\,Y_t(s,a)\given\Fullstate_t=\fullstate} \\
    &  = \sumsa r(s,a)\,Y_t(s,a) \\
    &  = y^*\rvec^\top + (y_t-y^*)\rvec^\top + (Y_t-y_t)\rvec^\top \\
    &  = y^*\rvec^\top + (x([N])-\statdist)\,\Cinvs\rvec^\top + \errrew(\fullstate),
\end{align*}
    Since $y^*  \rvec^\top = \rrel$ and $(x([N]) - \statdist)\,\Cinvs\,\rvec^\top = (x([N]) - \statdist)\,\gvec$ (the shift term in the definition of $\gvec$ does not contribute because $(x([N]) - \statdist)\,\vone^\top = 0$), this gives $r^\pi(\fullstate) = \rhat(x([N])) + \errrew(\fullstate)$.
    To bound $\errrew(\fullstate)=(Y_t - y_t) \rvec^\top$, we note that by \Cref{lem:wcmdp:olc-rounding}, we have $\norm{Y_t - y_t}_1 \leq  2|\sspa|(|\aspa|-1)/N$. Therefore,
    \begin{equation*}
        |\errrew(\fullstate)| \leq \rmax \norm{Y_t-y_t}_1
        \leq \frac{2\rmax |\sspa|(|\aspa|-1)}{N}.
    \end{equation*}

    Finally, we show $\norm{\gvec}_\infty \leq \Kg$. By construction, $\statdist\,\gvec = \statdist\,\Cinvs\,\rvec^\top - (\statdist\,\Cinvs\,\rvec^\top)(\statdist\,\vone^\top) = 0$. Hence, for each $s\in\sspa$,
    \begin{equation}\label{eq:wcmdp:gvec-entry}
        \gvec(s) = (e_s - \statdist)\,\gvec,
    \end{equation}
    where $e_s$ denotes the point mass at state $s$. If $e_s=\statdist$, then \eqref{eq:wcmdp:gvec-entry} gives $\gvec(s)=0$, so the desired bound holds. We therefore assume $e_s\ne\statdist$ below.
    To bound the right-hand side of \eqref{eq:wcmdp:gvec-entry}, we let $\theta \triangleq \min\bigl(1, \eta/\norm{e_s - \statdist}_\umat\bigr)$ and $v \triangleq (1-\theta)\,\statdist + \theta\,e_s$, so that $v \in \simplex(\sspa)$ and $\norm{v - \statdist}_\umat \leq \eta$. By \Cref{lem:wcmdp:linear-feasibility}, $y(v) \triangleq y^* + (v - \statdist)\,\Cinvs$ is entrywise non-negative; since its state-marginals equal $v$ by \eqref{eq:wcmdp:active-system}, $y(v)$ is a probability distribution over $\sspa\times\aspa$. Therefore, $\rhat(v) = y(v)\,\rvec^\top \in [-\rmax, \rmax]$. Combining this with $\rhat(v) - \rrel = \theta\,(e_s - \statdist)\,\gvec$ and $|\rrel| \leq \rmax$, we get
    \[
\begin{aligned}
\abs{\gvec(s)} &= \frac{\abs{\rhat(v)-\rrel}}{\theta} \leq \frac{2\rmax}{\theta} \\
    &\leq 2\rmax\max\Bigl(1,\frac{\norm{e_s-\statdist}_\umat}{\eta}\Bigr) \\
    &\leq 2\rmax\Bigl(1+\frac{\sqrt{2}\,\lamu^{1/2}}{\eta}\Bigr) = \Kg,
\end{aligned}
\]
    where in the last inequality, we use the facts that $\norm{e_s - \statdist}_\umat \leq \lamu^{1/2}\,\norm{e_s - \statdist}_2$ and $\norm{e_s - \statdist}_2^2 \leq \norm{e_s - \statdist}_1 \norm{e_s - \statdist}_\infty \leq 2$. \halmos
\end{proof}

To prove the global reward bound \eqref{eq:wcmdp:warmup-reward-bound}, we first consider $\norm{x([N])-\statdist}_\umat\leq\eta$. Since $\rrel-\rhat(v)=(\statdist-v)\gvec$, \Cref{lem:wcmdp:inst-reward} gives
\[
\rrel - r^\pi(\fullstate)
    \leq (\statdist - x([N]))\gvec + \frac{2\rmax|\sspa|(|\aspa|-1)}{N}.
\]
The first term is at most $\Kg\norm{\statdist-x([N])}_1\leq\sqrt{|\sspa|}\Kg\norm{\statdist-x([N])}_\umat$, because $\umat\succeq I$. Outside this neighborhood, $\rrel-r^\pi(\fullstate)\leq 2\rmax\leq(2\rmax/\eta)\norm{\statdist-x([N])}_\umat$. Choosing $K_0\geq\max\{\sqrt{|\sspa|}\Kg,2\rmax/\eta\}$, we therefore get \eqref{eq:wcmdp:warmup-reward-bound} in both cases.

\subsection{Proof of \hcref{lem:wcmdp:drift-Vds-truncate-term}}\label{sec:wcmdp:pf-drift-truncate}

\wcmdpdrifttruncate*

\begin{proof}
We use the abbreviation that $a \triangleq \Vds(\fullstate)$, $b \triangleq \Vds(\fullstate')$, and $r_0 \triangleq (1-\rhoFinal)\threshbar/2$.
The key algebraic identity
\begin{equation}\label{eq:trunc-key-identity}
    b - \frac{\threshbar}{2} = \Bigl(b - \rhoFinal a - r_0\Bigr) + \rhoFinal\Bigl(a - \frac{\threshbar}{2}\Bigr),
\end{equation}
together with subadditivity $(x+y)^+ \leq x^+ + y^+$, yields
\begin{equation}\label{eq:trunc-unified}
\begin{aligned}
& \Bigl(b-\frac{\threshbar}{2}\Bigr)^{\!+}-\Bigl(a-\frac{\threshbar}{2}\Bigr)^{\!+} \\
    & \quad \leq \Bigl(b-\rhoFinal a-r_0\Bigr)^{\!+} \\
    & \qquad - (1-\rhoFinal)\Bigl(a-\frac{\threshbar}{2}\Bigr)^{\!+}.
\end{aligned}
\end{equation}
We derive the claim from \eqref{eq:trunc-unified} by cases on $a$.

\medskip
\noindent\textbf{Case 1: $a \leq \threshbar$.}
Since $(1-\rhoFinal)(a-\threshbar/2)^+ \geq 0$, \eqref{eq:trunc-unified} gives
\[
    \Bigl(b-\frac{\threshbar}{2}\Bigr)^{\!+} - \Bigl(a-\frac{\threshbar}{2}\Bigr)^{\!+}
    \leq  \Bigl(b - \rhoFinal a - r_0\Bigr)^{\!+},
\]
which matches the RHS of \Cref{lem:wcmdp:drift-Vds-truncate-term} because $\indibrac{a>\threshbar}=0$.

\medskip
\noindent\textbf{Case 2: $a > \threshbar$.}
Then $(a-\threshbar/2)^+ = a - \threshbar/2 > \threshbar/2$, so $(1-\rhoFinal)(a-\threshbar/2)^+ > (1-\rhoFinal)\threshbar/2 = r_0$.
Hence \eqref{eq:trunc-unified} gives
\[
    \Bigl(b-\frac{\threshbar}{2}\Bigr)^{\!+} - \Bigl(a-\frac{\threshbar}{2}\Bigr)^{\!+}
    \leq  \Bigl(b - \rhoFinal a - r_0\Bigr)^{\!+} - r_0,
\]
which matches the RHS of \Cref{lem:wcmdp:drift-Vds-truncate-term} because $\indibrac{a>\threshbar}=1$.
\end{proof}

\subsection{Proof of \hcref{lem:wcmdp:drift-local-approx-term}}\label{sec:wcmdp:pf-drift-local}

\wcmdpdriftlocal*

\begin{proof}[Proof of \Cref{lem:wcmdp:drift-local-approx-term}]
    Using $(I-\Phi)Q=I$, $\Delta\locapprox$ admits the decomposition
    \begin{equation}\label{eq:wcmdp:drift-locapprox:step-1}
\begin{aligned}
& \Delta\locapprox(\fullstate) \\
    & \quad = \Bigl(\Ebig{(\statdist-X_{t+1}([N]))\,Q\gvec\givenbig\Fullstate_t=\fullstate} \\
    & \qquad - (\statdist-x([N]))\,\Phi Q\gvec\Bigr) \\
    & \qquad - \bigl(\rrel-\rhat(x([N]))\bigr).
\end{aligned}
\end{equation}

    The first term on the right-hand side of \eqref{eq:wcmdp:drift-locapprox:step-1} can be bounded as
    \begin{equation}\label{eq:wcmdp:drift-locapprox-claim}
\begin{aligned}
& \Ebig{(\statdist-X_{t+1}([N]))\,Q\gvec\givenbig\Fullstate_t=\fullstate} \\
    & \quad - (\statdist-x([N]))\,\Phi Q\gvec \\
    & \quad \leq 2\lamq\Kg\,\indibrac{\Db_t\neq[N]} \;+\; \frac{\tilde K}{N}.
\end{aligned}
\end{equation}
    where $\tilde{K} = 2|\sspa|(|\aspa|-1)\,\norm{Q}_\infty\,\Kg$. Here the term $2\lamq\Kg$ comes from bounding the left-hand side by $2\lamq\norm{\gvec}_\infty$ on the event $\Db_t \neq [N]$, and the term $\tilde{K}/N$ comes from bounding the rounding error term in \Cref{lem:wcmdp:olc-transition} by $\norm{Y_t - y_t}_1\,\norm{Q}_\infty\norm{\gvec}_\infty$ on the event $\Db_t = [N]$; both bounds use the property $\norm{\gvec}_\infty \leq \Kg$ from \Cref{lem:wcmdp:inst-reward}.

    For the second term on the right-hand side of \eqref{eq:wcmdp:drift-locapprox:step-1}, we have
    \begin{equation}\label{eq:wcmdp:drift-locapprox:rhat-conversion}
\begin{aligned}
& -(\rrel-\rhat(x([N]))) \\
    & \quad \leq -(\rrel-r^\pi(\fullstate)) \\
    & \qquad + (2\rmax+2\Kg)\,\indibrac{\norm{x([N])-\statdist}_\umat>\threshbar} \\
    & \qquad + \frac{2\rmax|\sspa|(|\aspa|-1)}{N}
\end{aligned}
\end{equation}
    where the term $2\rmax|\sspa|(|\aspa|-1) / N$ arises from the application of \Cref{lem:wcmdp:inst-reward}, which states that $|\rhat(x([N])) - r^\pi(\fullstate)| \leq 2\rmax|\sspa|(|\aspa|-1)/N$ when $\Db_t = [N]$; the factor $(2\rmax + 2\Kg)$ on the event $\norm{x([N]) - \statdist}_\umat > \threshbar$ comes from $|\rhat(v)| \leq |\rrel| + \norm{v - \statdist}_1 \norm{\gvec}_\infty \leq \rmax + 2\Kg$ for all $v\in\simplex(\sspa)$, together with $|r^\pi(\fullstate)| \leq \rmax$. 

    Substituting \eqref{eq:wcmdp:drift-locapprox-claim} and \eqref{eq:wcmdp:drift-locapprox:rhat-conversion} into \eqref{eq:wcmdp:drift-locapprox:step-1} and using $\indibrac{\Db_t \neq [N]} = \indibrac{\norm{x([N]) - \statdist}_\umat > \threshbar}$, we obtain
    \[
\begin{aligned}
& \Delta\locapprox(\fullstate) \\
    & \quad \leq -(\rrel-r^\pi(\fullstate)) + \frac{\Kla}{N} \\
    & \qquad + (2\lamq\Kg+2\rmax+2\Kg)\,\indibrac{\norm{x([N])-\statdist}_\umat>\threshbar},
\end{aligned}
\]
    with $\Kla \triangleq \tilde K + 2\rmax|\sspa|(|\aspa|-1) = 2|\sspa|(|\aspa|-1)\,(\norm{Q}_\infty\Kg + \rmax)$.
\end{proof}

\subsection{Supporting Lyapunov calculations}\label{app:wcmdp:lyapunov-calculations}

We recall that $\Vds(\fullstate)$ is defined as
\begin{equation}\label{eq:two-set:full-lyapunov}
\begin{aligned}
\Vds(\fullstate) &= \hu(x,\Db) + \hw(x,\Da) \\
    &\quad + \Lds(1-m(\Db)).
\end{aligned}
\end{equation}

\medskip
\textbf{Proof of \eqref{eq:wcmdp:Vds-dom}.}
By \Cref{lem:wcmdp:hw-hu-properties}, $\hu(x, D)$ is Lipschitz continuous in $D$, so $\hu(x, \Db)$ changes by at most $2\lamu^{1/2}(1-m(\Db))$ when we replace $\Db$ with $[N]$. Consequently, 
\[
\begin{aligned}
\Vds(\fullstate) &\geq \hu(x,\Db) + 2\lamu^{1/2}(1-m(\Db)) \\
    &\geq \hu(x,[N]) \\
    &= \norm{x([N])-\statdist}_\umat.
\end{aligned}
\]

\medskip
\textbf{A preliminary bound for proving \eqref{eq:wcmdp:Vds-mean} and \eqref{eq:wcmdp:Vds-tail}.}
We prove the following bound for constants $\rhoFinal\in(0,1)$ and $\Kdrift>0$, both independent of $N$ and specified below:
\begin{align}
    & \Vds(\Fullstate_{t+1}) - \rhoFinal\Vds(\Fullstate_t) \nonumber \\
    & \quad \leq \big(\hu(X_{t+1},\Db_t) - \rhou\hu(X_t,\Db_t)\big) \nonumber \\
    & \qquad + \big(\hw(X_{t+1},\Da_t) - \rhow\hw(X_t,\Da_t)\big) \label{eq:wcmdp:calc:v-diff-final-1} \\
    & \qquad + 4\big(\lamu^{1/2}+\lamw^{1/2}\big)m(\Db_t\backslash\Db_{t+1}) \nonumber \\
    & \qquad + \frac{\Kdrift}{N}. \label{eq:wcmdp:calc:v-diff-final-2}
\end{align}
To show this, we start with the decomposition:
\begin{align}
    & \Vds(\Fullstate_{t+1})-\Vds(\Fullstate_t) \nonumber \\
    & \quad = \big(\Vds(X_{t+1},\Db_t,\Da_t)-\Vds(X_t,\Db_t,\Da_t)\big) \label{eq:two-set:thm:state-change-diff-term} \\
    & \qquad + \big(\Vds(X_{t+1},\Db_{t+1},\Da_{t+1})-\Vds(X_{t+1},\Db_t,\Da_t)\big). \label{eq:two-set:thm:set-change-diff-term}
\end{align}
We refer to the difference term in right-hand side of \eqref{eq:two-set:thm:set-change-diff-term} the \emph{set-update term}, and the difference term in \eqref{eq:two-set:thm:state-change-diff-term} the \emph{state-transition term}. 
We calculate these terms separately. 

For the state-transition term in \eqref{eq:two-set:thm:state-change-diff-term}, it follows directly from the definition of $\Vds$ that 
\begin{align}
    & \Vds(X_{t+1},\Db_t,\Da_t) - \Vds(X_t,\Db_t,\Da_t) \nonumber \\
    & \quad \leq \big(\hu(X_{t+1},\Db_t) - \hu(X_t,\Db_t)\big) \nonumber \\
    & \qquad + \big(\hw(X_{t+1},\Da_t) - \hw(X_t,\Da_t)\big) \label{eq:two-set:thm:state-transition-term-1}
\end{align}

For the set-update term in \eqref{eq:two-set:thm:set-change-diff-term}, we can derive a bound fully in terms of the sizes of sets, using the Lipschitz continuity of $\hw(x, D)$ and $\hu(x, D)$ with respect to $D$ given in \eqref{eq:wcmdp:hw:lipschitz} of \Cref{lem:wcmdp:hw-hu-properties}:
\begin{align}
    & \Vds(X_{t+1},\Db_{t+1},\Da_{t+1}) - \Vds(X_{t+1},\Db_t,\Da_t) \nonumber \\
    & \quad = \hu(X_{t+1},\Db_{t+1}) - \hu(X_{t+1},\Db_t) \nonumber \\
    & \qquad + \hw(X_{t+1},\Da_{t+1}) - \hw(X_{t+1},\Da_t) \nonumber \\
    & \qquad - \Lds\big(m(\Db_{t+1})-m(\Db_t)\big) \nonumber \\
    & \quad \leq 2\lamu^{1/2}\big(m(\Db_{t+1}\backslash\Db_t) + m(\Db_t\backslash\Db_{t+1})\big) \nonumber \\
    & \qquad + 2\lamw^{1/2}\big(m(\Da_{t+1}\backslash\Da_t) + m(\Da_t\backslash\Da_{t+1})\big) \label{eq:two-set:thm:set-update-term-1} \\
    & \qquad - \Lds\big(m(\Db_{t+1}\backslash\Db_t)-m(\Db_t\backslash\Db_{t+1})\big).
\end{align}
Next, we show that 
\begin{equation}\label{eq:two-set:thm:da-diff-bound}
\begin{aligned}
& m(\Da_{t+1}\backslash\Da_t) + m(\Da_t\backslash\Da_{t+1}) \\
    & \quad \leq 2m(\Db_{t+1}\backslash\Db_t) \\
    & \qquad - \beta\big(m(\Db_{t+1})-m(\Db_t)\big) + \frac{1}{N}.
\end{aligned}
\end{equation}
We discuss based on whether $\Da_{t+1} \supseteq \Da_{t} \backslash \Db_{t+1}$ or $\Da_{t+1} \subseteq \Da_{t} \backslash \Db_{t+1}$. 
\begin{itemize}
    \item If $\Da_{t+1} \supseteq \Da_{t} \backslash \Db_{t+1}$, because $\Da_t$ is disjoint from $\Db_t$, it is not hard to see that
    \begin{equation}\label{eq:two-set:thm:set-update-term-2}
\begin{aligned}
\Da_t\backslash\Da_{t+1} &\subseteq \Db_{t+1}\backslash\Db_t, \\
    m(\Da_t\backslash\Da_{t+1}) &\leq m(\Db_{t+1}\backslash\Db_t).
\end{aligned}
\end{equation}
    Moreover, by the definition of $\Da_t$, $m(\Da_t) \geq \beta (1-m(\Db_{t})) - 1/N$ and  $m(\Da_{t+1}) \leq \beta (1-m(\Db_{t+1}))$, so 
    \begin{align}
    & m(\Da_{t+1}\backslash\Da_t) \nonumber \\
    & \quad = m(\Da_t\backslash\Da_{t+1}) + m(\Da_{t+1}) - m(\Da_t) \nonumber \\
    & \quad \leq m(\Db_{t+1}\backslash\Db_t) \nonumber \\
    & \qquad - \beta\big(m(\Db_{t+1})-m(\Db_t)\big) + \frac{1}{N}. \label{eq:two-set:thm:set-update-term-3}
\end{align}
    Combining \eqref{eq:two-set:thm:set-update-term-2} and \eqref{eq:two-set:thm:set-update-term-3}, we get \eqref{eq:two-set:thm:da-diff-bound}.
    \item If $\Da_{t+1} \subseteq \Da_{t} \backslash \Db_{t+1}$, we have $m(\Da_{t+1} \backslash \Da_{t}) = 0$ and
    \begin{align}
    m(\Da_t\backslash\Da_{t+1}) 
    &  = m(\Da_t)-m(\Da_{t+1}) \nonumber \\
    &  = \frac{1}{N}\Bigl(\lfloor\beta N(1-m(\Db_t))\rfloor \nonumber \\
    & \quad - \lfloor\beta N(1-m(\Db_{t+1}))\rfloor\Bigr) \nonumber \\
    &  \leq \beta\big(m(\Db_{t+1})-m(\Db_t)\big) + \frac{1}{N}, \nonumber
\end{align}
    which implies \eqref{eq:two-set:thm:da-diff-bound} because $m(\Db_{t+1}\backslash\Db_t) \geq m(\Db_{t+1}) - m(\Db_t)$ and $\beta < 1$. 
\end{itemize}
Plugging \eqref{eq:two-set:thm:da-diff-bound} into \eqref{eq:two-set:thm:set-update-term-1}
and rearranging the terms, we get an upper bound for the set-update term:
\begin{align}
    & \Vds(X_{t+1},\Db_{t+1},\Da_{t+1}) - \Vds(X_{t+1},\Db_t,\Da_t) \nonumber \\
    & \quad \leq 4\big(\lamu^{1/2}+\lamw^{1/2}\big)m(\Db_t\backslash\Db_{t+1}) \nonumber \\
    & \qquad + \frac{2\lamw^{1/2}}{N}, \label{eq:two-set:thm:set-update-term-4}
\end{align}
where the terms involving to $m(\Db_{t+1}\backslash \Db_t)$ have been canceled out. 

Substituting the above calculation into
\eqref{eq:two-set:thm:state-change-diff-term} and \eqref{eq:two-set:thm:set-change-diff-term}, we get
\begin{align}
    & \Vds(\Fullstate_{t+1}) - \Vds(\Fullstate_t) \nonumber \\
    & \quad \leq \big(\hu(X_{t+1},\Db_t)-\hu(X_t,\Db_t)\big) \nonumber \\
    & \qquad + \big(\hw(X_{t+1},\Da_t)-\hw(X_t,\Da_t)\big) \label{eq:two-set:thm:v-diff-term-1} \\
    & \qquad + 4\big(\lamu^{1/2}+\lamw^{1/2}\big)m(\Db_t\backslash\Db_{t+1}) \nonumber \\
    & \qquad + \frac{2\lamw^{1/2}}{N}. \label{eq:two-set:thm:v-diff-term-2}
\end{align}
By \Cref{lem:wcmdp:sufficient-coverage}, we have
\begin{align}
 \Vds(\Fullstate_t)
    &\leq \hu(X_t,\Db_t) + \hw(X_t,\Da_t) \nonumber \\
    & \qquad + \Lds\bigg(\KscOne\hw(X_t,\Da_t)+\frac{\KscTwo}{N}\bigg) \nonumber \\
    & \quad \leq K_{Vh}\cdot\big((1-\rhou)\hu(X_t,\Db_t) \nonumber \\
    & \qquad + (1-\rhow)\hw(X_t,\Da_t)\big) + \frac{\Lds\KscTwo}{N}, \nonumber
\end{align}
where we choose $K_{Vh}>1$, independent of $N$, so that $K_{Vh}(1-\rhou)\geq1$ and $K_{Vh}(1-\rhow)\geq1+\Lds\KscOne$. We set $\rhoFinal=1-1/K_{Vh}\in(0,1)$. Multiplying the preceding bound by $1-\rhoFinal=1/K_{Vh}$ gives
\begin{equation}\label{eq:v-h-bound}
\begin{aligned}
& (1-\rhoFinal)\Vds(\Fullstate_t) \\
    & \quad \leq (1-\rhou)\hu(X_t,\Db_t) + (1-\rhow)\hw(X_t,\Da_t) \\
    & \qquad + \frac{(1-\rhoFinal)\Lds\KscTwo}{N}.
\end{aligned}
\end{equation}
Adding \eqref{eq:v-h-bound} to \eqref{eq:two-set:thm:v-diff-term-1}--\eqref{eq:two-set:thm:v-diff-term-2} proves \eqref{eq:wcmdp:calc:v-diff-final-1}--\eqref{eq:wcmdp:calc:v-diff-final-2} with
\[
    \Kdrift\triangleq 2\lamw^{1/2}+(1-\rhoFinal)\Lds\KscTwo.
\]

\fi

\end{document}